\documentclass{article}
\usepackage[utf8]{inputenc}
\usepackage{fullpage,graphicx}
\usepackage{float}
\usepackage{amsmath,amsfonts,amsthm,amssymb,multirow,xcolor,mathtools}
\usepackage{algpseudocode,array,tcolorbox}
\usepackage{xcolor}
\usepackage{url}
\usepackage{tikz}
\usepackage{tikz-qtree}
\usepackage{bbm}
\usepackage{wrapfig}
\usepackage{caption}
\usepackage{subcaption}

\usepackage[numbers]{natbib}

\usepackage{thmtools, thm-restate}
\usepackage{proof-at-the-end}
\usepackage{csquotes}

\newtheorem{theorem}{Theorem}
\newtheorem{lemma}[theorem]{Lemma}

\def\messages{\mathcal{M}}
\def\prefixes{\mathcal{P}}
\def\types{\mathcal{Z}}

\def\data{\mathcal{D}}
\def\actions{\mathcal{A}}

\def\KL{\mathbf{d}_{\mathrm{KL}}}

\def\nn{\nonumber}

\def\1{\mathbbm{1}}
\def\E{\mathbb{E}}

\def\Pr{\mathbb{P}}

\def\hat{\widehat}

\DeclareMathOperator*{\argmax}{arg\,max}
\DeclareMathOperator*{\argmin}{arg\,min}

\newcommand{\Ic}{\mathcal{I}}
\newcommand{\Xc}{\mathcal{X}}
\newcommand{\Yc}{\mathcal{Y}}
\newcommand{\Zc}{\mathcal{Z}}

\newcommand{\Ec}{\mathcal{E}}

\newcommand{\Mc}{\mathcal{M}}

\newcommand{\Lc}{\mathcal{L}}

\definecolor{red}{RGB}{255,0,0}
\definecolor{blue}{RGB}{0,0,255}
\definecolor{green}{RGB}{0,255,0}
\definecolor{orange}{RGB}{255,165,0}
\definecolor{purple}{RGB}{128,0,128}
\definecolor{teal}{RGB}{0,128,128}

\usepackage[ruled,vlined,commentsnumbered,titlenotnumbered]{algorithm2e}

\usepackage{fancyvrb}

\makeatletter
\newcommand{\specificthanks}[1]{\@fnsymbol{#1}}
\makeatother

\usepackage{authblk}
\author[1]{Wanqiao Xu}
\author[2]{Shi Dong}
\author[2]{Xiuyuan Lu}
\author[2]{\\ Grace Lam}
\author[2]{Zheng Wen}
\author[1,2]{Benjamin Van Roy}
\affil[1]{Stanford University}
\affil[2]{Google DeepMind}

\title{RLHF and IIA: Perverse Incentives}

\begin{document}

\maketitle

\begin{abstract}
Existing algorithms for reinforcement learning from human feedback (RLHF) can incentivize responses at odds with preferences because they are based on models that assume independence of irrelevant alternatives (IIA). 
The perverse incentives induced by IIA hinder innovations on query formats and learning algorithms.
\end{abstract}

\section{Introduction}

Modern generative AIs ingest trillions of data bytes from the World Wide Web to produce a large pretrained model.  Trained to imitate what is observed, this model represents an agglomeration of behaviors, some of which are more or less desirable to mimic.  Further training through human interaction, even on fewer than a hundred thousand bits of data, has proven to greatly enhance usefulness and safety, enabling the remarkable AIs we have today.  This process of {\it reinforcement learning from human feedback} (RLHF) steers AIs toward the more desirable among behaviors observed during pretraining.

While AIs now routinely generate drawings, music, speech, and computer code, the text-based chatbot remains an emblematic artifact.  To produce a chatbot, starting with a pretrained language model, a prototypical approach to RLHF \citep{christiano2017deep,stiennon2020learning,ouyang2022training} progresses through several steps.  First, queries, each taking the form of a prompt and a pair of alternative responses, are presented to human annotators who each identify their favorite among a pair.  The annotated data is then used to train a reward model to score any response to a given prompt.  Finally, the language model is tuned to align its responses toward those that earn high reward.

The tremendous impact of RLHF in generative AI has sparked a flurry of research that aims to understand and improve the process.  Some propose alternative algorithms \citep{rafailov2023dpo,xu2023shattering,zhao2023slichf,hejna2023contrastive,dumoulin2023density}.  Others consider alternative query formats \cite{glaese2022improving,zhu2023principled,song2023preference,yuan2023rrhf,dong2023raft,rafailov2023dpo} for which annotators, rather than comparing only a pair of responses, are asked to choose from a longer list, or to rank-order.  Feedback can also be garnered from interactions with humans in their regular use of online chatbots.  RLHF research is continually growing in importance with the volume of human feedback data.

With all the effort and resources directed at RLHF, it is worth asking whether current algorithms rest on firm foundations.  Maybe not, as these algorithms are based on models that assume independence of irrelevant alternatives (IIA), which intuitively means that, when making a choice between two options, the introduction of a third option should not alter preferences between the original two.  As we will demonstrate, human preferences for text content violate IIA.
Even though this flaw is not pronounced when using the most common approach of fitting a reward model to pairwise comparison data, followed by tuning the language policy to optimize reward, it makes current RLHF approaches rigid.  Even simple tweaks to the query format or learning algorithm can lead to undesirable outcomes.

A simple experiment we will present in Section \ref{se:empirical} illustrates our point.  This experiment applies a standard reward learning approach \citep{christiano2017deep,stiennon2020learning,ouyang2022training}.  We first consider learning from queries that are each comprised of the prompt
\begin{displayquote}
\begin{tabular}{ll}
{\bf prompt} & {\it Did Oppenheimer win a Nobel Prize?}
\end{tabular}
\end{displayquote}
and a pair of responses, one generated by GPT-3.5 and the other by GPT-4.  The former generally produces more concise and the latter more informative responses.  For example, here are representative responses:
\begin{displayquote}
\begin{tabular}{ll}
{\bf GPT-3.5 response} & {\it No, Oppenheimer did not win the Nobel Prize.} \\
{\bf GPT-4 response} & {\it No, Robert Oppenheimer, often called the ``father of the atomic bomb''} \\
& {\it for his role in the Manhattan Project, did not win a Nobel Prize.}
\end{tabular}
\end{displayquote}
If a large majority of annotators prefer responses generated by GPT-4, the learned reward function correctly assigns higher scores to GPT-4 responses in independent test data.

A variation in which training queries include four rather than two responses reveals egregious behavior induced by the IIA assumption.  In particular, suppose that each query includes one response generated by GPT-3.5 and three by GPT-4.  Then, even if a large majority of annotators prefer responses generated by GPT-4, the learned reward function erroneously assigns higher scores to GPT-3.5 responses in test data.

The preceding example indicates that innocuous changes to the query format can cause standard RLHF pipelines to fail.  Innovating on RLHF algorithms can also result in undesirable behavior.
As an example, we will consider inclusive learning (IL) \cite{arumugam2022inclusive,xu2023shattering} and sequence likelihood calibration (SLiC) from human feedback \cite{zhao2023slichf}.
In Section \ref{se:il-slic}, we show that the IIA assumption exacts egregious behavior from IL and SLiC even on standard preference data.

We demonstrate in this paper how the IIA assumption imposes serious limitations on current RLHF approaches, hindering innovations on alternative query formats and learning algorithms.
The remainder of the paper explains issues and methods more deeply, with a goal to enhance understanding of flaws in current algorithms.  We develop this understanding through interpreting simulations of didactic models, an empirical study of data produced by GPT models, and theoretical results that corroborate the observed behavior.  These theoretical results establish that such behavior generalizes beyond specific instances.  We leave for future work the design of algorithms that alleviate flaws we identify here.

\section{Language Models, Messages, and Preference Data}
\label{se:definitions}

We begin by defining a few basic concepts that we will use when describing and analyzing algorithms.

\subsection{Language Models}

Let $\actions$ be a finite set, which we will refer to as an {\it alphabet}.  We refer to elements as {\it tokens}.  We denote the set of finite token sequences by $\actions^+$.  A {\it language model} is a function $\pi$ that, for each $x \in \actions^+$, specifies a probability mass function $\pi(\cdot|x)$ over $\actions$.  For example, if the alphabet is the set of ASCII characters, a language model provides a distribution $\pi(\cdot|(X_1, X_2, \ldots, X_t))$ conditioned on $t$ characters typed so far.  Sequentially sampling characters $X_{t+1} \sim \pi(\cdot|(X_1, \ldots, X_t))$ generates text.  A language model is alternately referred to as a {\it policy}, as it provides a rule for selecting a next token based on the preceding sequence.

\subsection{Messages}

There is a set $\messages \subseteq \actions^+$ of token sequences that we refer to as {\it messages}.  We interpret $\messages$ as the set of complete individual statements.  For example, in a text messaging application, a message might terminate with either the ASCII {\it enter} character or upon reaching a maximum allowed message length.

A language model generates a random message by sampling sequentially according to $X_{t+1} \sim \pi(\cdot|(X_1, \ldots, X_t))$, terminating at the first index $T$ for which $(X_1, \ldots, X_T) \in \messages$.  For any token sequence $x = (x_1,\ldots,x_T) \in \actions^+$, we use the notation $x_{t_1:t_2}$ to denote the subsequence $(x_{t_1}, \ldots, x_{t_2})$.  Hence, we can use the abbreviated form $X_{t+1} \sim \pi(\cdot | X_{1:t})$.  We assume that, for any language model $\pi$, any token sequence generated sequentially in this manner terminates with probability one.  Through this procedure, a language model $\pi$ samples messages from a distribution $P_\pi$, defined by
\begin{equation}
\label{eq:message-probability}
P_\pi(x) = \prod_{t=1}^{|x|} \pi(x_t|x_{1:t-1}),    
\end{equation}
for $x \in \messages$, where $|x|$ denotes the length of sequence $x$.  For a set of messages $\mathcal{S}$, we use $P_\pi(\mathcal{S})$ to denote $\sum_{x\in \mathcal{S}} P_\pi(x)$.

\subsection{Prefixes}

We refer to a token sequence $x$ as a {\it prefix} if $x_{1:t} \notin \messages$ for all $t = 1,\ldots,|x|$.  In our aforementioned text messaging example, a prefix would be any text absent the {\it enter} character and of length less than the maximum allowed.  Let $\prefixes$ denote the set of prefixes.  While a language model $\pi$ specifies probabilities $\pi(\cdot|x)$ for all token sequences $x \in \actions^+$, only those for which $x \in \prefixes$ are consequential to message generation.

\subsection{Preference Data}

Each {\it preference datum} is a pair $(\Yc, y)$.  The first item is a set $\Yc \subseteq \messages$ of alternative messages, while the second is a choice $y \in \Yc$.  For example, each message in $\mathcal{Y}$ could be a poem, with the choice $y$ indicating an individual's favorite among them.  RLHF algorithms we present use a {\it preference dataset} $\data$.  Each element is a preference datum, and the number of alternative messages can vary across elements.  Note that $\Yc$ and $\data$ are technically multisets because elements can repeat.  However, we will refer to these informally as sets.

\subsection{What About Prompting?}

Language models typically generate a message in response to a prompt made up of previous messages.  In that context, each preference datum expresses a choice between responses to a prompt.  Except for our empirical study of Section \ref{se:empirical}, we omit prompting from our formulation and analysis, because that would only complicate the discussion without contributing to insight.  Examples, algorithms, and results we present can easily be extended to treat prompting.

\section{Choice Models}

Given a set $\Yc \subseteq \messages$ of alternatives, a choice model generates a random element $Y \in \Yc$. We interpret $Y$ to be the choice made by a random individual.  Choice models we consider are each defined with respect to a message set $\messages$ and a triple $(\types, p, r)$, consisting of a set $\types$ of individual types, a type distribution $p$, and a reward function $r$.  We will define these terms and then present examples.  Choices in the first and second examples, but not the third example, satisfy IIA.  After presenting these models, we discuss their relation to IIA and why human preferences for text violate IIA.

\subsection{Individual Types}

Individuals are distinguished by type.  Let $\Zc$ denote the set of types.  The type of a random individual, $Z$, is distributed according to a type distribution $p$.  If the set of types $\Zc$ is finite, this distribution assigns a probability $p(z)$ to each $z \in \types$.

\subsection{Reward Functions}

A reward function $r$ expresses individual preferences between elements of a message set $\messages$.  In particular, an individual of type $z \in \types$ prefers a message $x$ to $x'$ if and only if $r(x|z) \geq r(x'|z)$.

To illustrate with an example, two elements $z, z' \in \mathcal{Z}$ could express detailed characteristics of particular individuals, one of whom resides in China and the other in the United States.  These individuals have different tastes and needs -- a message $x$ that is helpful or interesting to one may not be to the other.  The distinct preferences are expressed by differing reward assignments $r(x|z) \neq r(x|z')$.

\subsection{Choice Probabilities}
\label{se:choice probabilities}

A choice model $(\types, p, r)$ expresses how random individuals make choices.  The model posits that, when presented with the set $\Yc$ of alternatives, an individual of type $z \in \Zc$ samples their choice $Y$ uniformly from $\argmax_{y \in \Yc} r(y|z)$.  This implies choice probabilities
\begin{equation}
\label{eq:choice-probs}
\Pr(Y = y|\Yc) = \E\left[\frac{\1(y \in \argmax_{x \in \Yc} r(x|Z))}{|\argmax_{x \in \Yc} r(x|Z)|}\right],
\end{equation}
where $Z$ is sampled from $p$.  In particular, $\Pr(Y = y|\Yc)$ is the probability that a randomly sampled individual presented with $\Yc$ will choose $y \in \Yc$.

\subsection{Example 1: Logit Models}
\label{se:logit}

The standard logit model can be expressed as a choice model $(\types, p, r)$.
The set $\mathcal{Z}$ of types is comprised of functions that map $\messages$ to $\Re_+$.  Hence, the type of a random individual is expressed by a random function $Z$.  For each $x \in \messages$, $Z(x)$ is distributed as an independent standard Gumbel.  For some fixed {\it base reward function} $\overline{r}:\actions^+ \rightarrow \Re$, let $r(x|z) = \overline{r}(x) + z(x)$.  Given this reward function and type distribution, sampling a choice $Y$ as described above implies choice probabilities governed by the standard logit model with rewards specified by the base reward function $\overline{r}$ \cite{luce1965preferences}.  In particular, $(\types, p, r)$ gives rise to choice probabilities
\begin{equation}
\label{eq:logit}
\Pr(Y = y| \Yc) = \frac{e^{\overline{r}(y)}}{\sum_{y'\in\Yc} e^{\overline{r}(y')}}.
\end{equation}
Conversely, only the choice model $(\types, p, r)$ that we have specified or one with a reward function that differs from $r$ by a constant produces these choice probabilities \cite{mcfadden1974preferences}.

As an example, consider a language model $\pi$ that recommends recipes for dinner.  For each recipe $y$, $\overline{r}(y)$ indicates its nominal desirability -- no individual attributes any less reward to $y$.  But for a specific individual of type $z$, a positive perturbation $z(y)$ is added to reflect their idiosyncratic preferences.  Presented with a set $\Yc$ of recipes, the individual selects the recipe $y \in \Yc$ that maximizes $\overline{r}(y) + z(y)$.  In this way, $z$ induces variation in how different individuals make their selection.  The Gumbel distribution enjoys elegant mathematical structure: if, for each recipe $y$, $z(y)$ is an independent sample from the standard Gumbel then the probability that a random individual selects $y$ over the remaining recipes in $\Yc$ is given by \eqref{eq:logit}.  

\subsection{Example 2: Soft Choice Models}
\label{se:soft-choice-model}

Soft choice models generalize on the logit.  Similarly with the choice model formalism defined in Equation \eqref{eq:choice-probs} -- which we alternately refer to as a {\it hard choice model} -- a soft choice model is specified by a triple $(\types, p, r)$.  However, unlike a hard choice model, choice probabilities are given by
\begin{equation}
\label{eq:soft-choice-probs}
\Pr(Y = y| \Yc) = \E\left[\frac{e^{r(y|Z)}}{\sum_{y'\in\Yc} e^{r(y'|Z)}}\right],
\end{equation}
instead of (\ref{eq:choice-probs}).

A soft choice model $(\types, p, r)$ can be expressed as a hard choice model $(\types', p', r')$.  By this we mean that choice probabilities $\Pr(Y=y|\Yc)$ implied by the two models are identical.  The type set $\Zc'$ is the product set $\Zc \times \Re_+^\messages$.  The type of a random individual is a pair $Z' = (Z,W)$.  The first item $Z$ is sampled from $p$.  The second item comprises of an independent standard Gumbel sample $W(x)$ for each $x \in \messages$.  The reward function is given by $r'(x|z') = r(x|z) + w(x)$.  It is easy to see that this hard choice model $(\types', p', r')$ produces the same choice probabilities (\ref{eq:soft-choice-probs}) as the soft choice model $(\types, p, r)$.

Note that a logit model can be expressed as a soft choice model.  Consider a logit model expressed as a hard choice model $(\types, p, r)$.  Define a soft choice model $(\types', p', r')$ with a single type $\types = \{0\}$, a trivial type distribution for which $p(0) = 1$, and a reward function $r'(x|0) = r(x|0)$. This soft choice model implies choice probabilties
$$\Pr(Y = y| \Yc) = \frac{e^{r'(y|0)}}{\sum_{y'\in\Yc} e^{r'(y'|0)}},$$
which are equivalent to the logit choice probabilities (\ref{eq:logit}) if we take the base reward function to be $\overline{r}(x) = r(x|0)$.

\subsection{Example 3: Dichotomy Models}
\label{se:dichotomy}

Our next model class is crafted to illustrate in a transparent manner perverse incentives induced by current RLHF algorithms. There are two individual types $\mathcal{Z} = \{1,2\}$, each of which prefers one of two categories of messages, identified by nonempty sets $\messages_1$ and $\messages_2$ for which $\messages_1 \cap \messages_2 = \emptyset$ and $\messages_1 \cup \messages_2 = \messages$.  The reward function takes the form
$$r(x|z) = \left\{\begin{array}{ll}
1 \qquad & \text{if } x \in \messages_z \\
0 \qquad & \text{otherwise.}
\end{array}
\right.$$
Note that reward depends on the message $x$ only through its membership in $\messages_1$ or $\messages_2$.  In other words, each individual is indifferent between any two messages of the same type.

The type set $\Zc$ and reward function $r$ we have defined, together with any type distribution $p$, constructs a choice model $(\Zc, p, r)$.  According to this choice model, when presented with alternatives $\mathcal{Y}$, the choice $Y$ made by an individual of type $Z$ is sampled uniformly from $\Yc \cap \messages_Z$.  Hence,
\begin{equation}
\label{eq:dichotomy-choice}
\Pr(Y=y|\Yc) = \E\left[\frac{\1_{\messages_Z}(y)}{|\Yc \cap \messages_Z|}\right],
\end{equation}
where $Z$ is sampled from $p$.

A hypothetical sort of homophily serves to illustrate a context where this sort of model may be relevant.  Suppose a language model $\pi$ generates political commentary, with $\messages_1$ and $\messages_2$ conveying conservative versus liberal views.  In this hypothetical world, there are two types of individuals -- conservatives and liberals -- labeled $1$ and $2$.  Each prefers commentary aligned with their predisposition.  There is no reward for opposing views.

\subsection{Independence of Irrelevant Alternatives}

Consider three messages $\messages = \{\mathrm{dog}, \mathrm{cat}, \mathrm{feline}\}$ that each guess an individual's favorite pet.  Consider a logit choice model, as described in Section \ref{se:logit}, that assigns equal base rewards $r(\mathrm{dog}) = r(\mathrm{cat}) = r(\mathrm{feline})$ to indicate that a random individual is as likely to prefer dogs or cats.  Presented with two alternatives $\Yc = \{\mathrm{dog}, \mathrm{cat}\}$, the logit model produces choice probabilities
$\Pr(Y=\mathrm{dog} | \Yc) = \Pr(Y=\mathrm{cat} | \Yc) = 1/2$.
On the other hand, with three alternatives $\Yc = \{\mathrm{dog}, \mathrm{cat}, \mathrm{feline}\}$,
$\Pr(Y=\mathrm{dog} | \Yc) = \Pr(Y=\mathrm{cat} | \Yc) = \Pr(Y=\mathrm{feline} | \Yc) = 1/3$.
This of course makes no sense, because {\it feline} is synonymous to {\it cat}, so presenting it as an alternative ought not reduce the probability that an individual prefers dogs.

The aforementioned implausible implication of a logit model derives from its assumption of independence of irrelevant alternatives (IIA).  While this notion was formalized in earlier work \cite{arrow1951social}, the treatment of \cite{luce1959individual} best suits our usage.  In that treatment, IIA indicates that the ratio of any two choice probabilities does not vary with the set of alternatives.  More formally, for all messages $x,x' \in \messages$ and sets $\Xc$ and $\Xc'$ such that $x,x' \in \Xc \cap \Xc'$, 
\begin{equation}
\label{eq:IIA}
\frac{\Pr(Y = x | \Yc = \Xc)}{\Pr(Y = x' | \Yc = \Xc)} = \frac{\Pr(Y = x | \Yc = \Xc')}{\Pr(Y = x' | \Yc = \Xc')}.
\end{equation}
This property implies that if cat-lovers are indifferent and randomly choose between {\it cat} and {\it feline} then the introduction of {\it feline} as an alternative reduces not only the probability that a random individual selects {\it cat} but also {\it dog}.  In particular, (\ref{eq:IIA}) specializes to
$$\frac{\Pr(Y = \mathrm{dog} | \Yc = \{\mathrm{dog}, \mathrm{cat}\})}{\Pr(Y = \mathrm{cat} | \Yc = \{\mathrm{dog}, \mathrm{cat}\})} = \frac{\Pr(Y = \mathrm{dog} | \Yc = \{\mathrm{dog}, \mathrm{cat}, \mathrm{feline}\})}{\Pr(Y = \mathrm{cat} | \Yc = \{\mathrm{dog}, \mathrm{cat}, \mathrm{feline}\})}.$$

Our dichotomy model of Section \ref{se:dichotomy} relaxes IIA and can generate more realistic pet choices, as we will now explain.  Take the individual types $\mathcal{Z} = \{1,2\}$ to be dog and cat lovers, with type probabilities $p_*(1) = p_*(2) = 1/2$, and rewards $r_1 = r_2 = 1$.  Then, presented with two alternatives $\Yc = \{\mathrm{dog}, \mathrm{cat}\}$, the dichotomy model produces the same choice probabilities as the logit model:
$\Pr(Y=\mathrm{dog} | \Yc) = \Pr(Y=\mathrm{cat} | \Yc) = 1/2$.
On the other hand, with three alternatives $\Yc = \{\mathrm{dog}, \mathrm{cat}, \mathrm{feline}\}$,
$\Pr(Y=\mathrm{dog} | \Yc) = 1/2$, while $\Pr(Y=\mathrm{cat} | \Yc) = \Pr(Y=\mathrm{feline} | \Yc) = 1/4$.  Hence, unlike the logit model, with the dichotomy model the probability that a random individual chooses {\it dog} appropriately remains unchanged.

\subsection{IIA for Language}

Human preferences for text messages do not satisfy IIA.  This is because, for any given message, others can be virtually equivalent.  Our preceding example with three messages -- {\it dog}, {\it cat}, and {\it feline} -- illustrates this.  The introduction of {\it feline} as an alternative should not reduce the probability that an individual chooses {\it dog}.

In spite of this, current RLHF algorithms are based on models that assume IIA.  In particular, preference estimates used in RLHF satisfy IIA even though human preferences expressed through data do not.  These estimates can give rise to perverse incentives to generate undesirable messages.  We will discuss in Section \ref{se:perverse-incentives} potentially egregious consequences.  But first, we review current RLHF algorithms.

\section{RLHF Algorithms}
\label{se:algorithms}

An RLHF algorithm, given a {\it base language model} $\overline{\pi}$ and {\it preference data} $\data$, produces a {\it fine-tuned language model} $\hat{\pi}$.  In practice, the base language model is typicaly a pretrained model that has undergone some supervised fine-tuning.  Figure \ref{fig:rlhf} illustrates the RLHF algorithm interface.  In this section, we describe several instances of RLHF algorithms.  

\begin{figure}[htb]
\begin{center}
\includegraphics[scale=0.25]{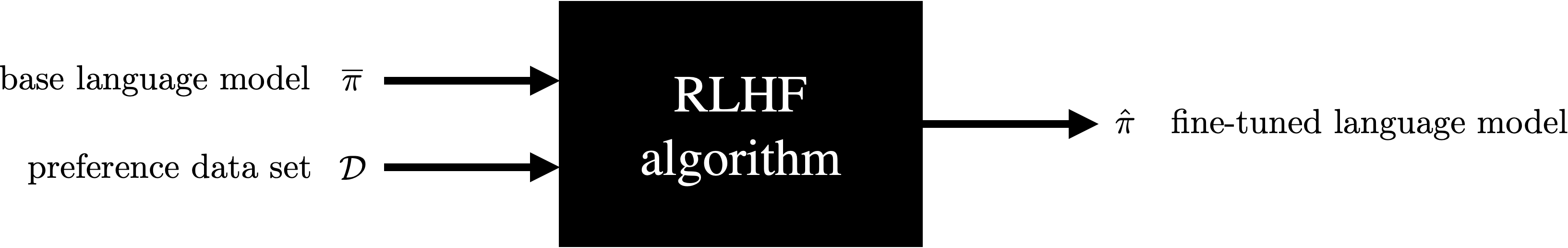}
\caption{RLHF algorithm interface.}
\label{fig:rlhf}
\end{center}
\end{figure}

\subsection{Reward Learning Followed by Policy Optimization}

A standard approach to RLHF first fits a logit model to preference data and then tunes a language model to optimize reward \cite{stiennon2020learning,ouyang2022training}.  We will refer to this as RLPO, which stands for {\it reward learning followed by policy optimization}.  RLPO uses a reward model $\hat{r}_\psi$, which is parameterized by a vector $\psi$.  This reward model maps $\messages$ to $\Re$ and is interpreted as expressing approximate soft choice probabilities $\Pr(Y = y| \Yc) \approx e^{\hat{r}_\psi(y)} / \sum_{y'\in\Yc} e^{\hat{r}_\psi(y')}$.  RLPO tunes parameters to minimize cross-entropy loss on preference data:
\begin{equation}
\label{eq:reward-loss}
\mathcal{L}_\mathrm{reward}(r) = - \sum_{(\Yc, y) \in \data} \ln\left(\frac{e^{r(y)}}{\sum_{y'\in\Yc} e^{r(y')}}\right).
\end{equation}
The reward function $r_\psi$ could, for example, be linear in a vector message embedding $\phi(x)$, taking the form $\hat{r}_\psi(x) = \psi^\top \phi(x)$.  Alternatively, $\hat{r}_\psi$ could be a neural network with parameters $\psi$.

For any $r:\messages\rightarrow\Re$ and language model $\pi$, we define a policy loss function
\begin{equation}
\label{eq:policy-optimization}
\mathcal{L}_\mathrm{policy}(\pi | r) = - |\data| \sum_{x \in \messages} P_\pi(x) r(x) + \beta \KL(P_\pi\|P_{\overline{\pi}}),
\end{equation}
for some scalar hyperparameter $\beta > 0$.  The first term of this loss function expresses expected reward while the second regularizes toward the base language model.

A fine-tuned language model is obtained via two steps.  The first minimizes $\mathcal{L}_\mathrm{reward}$ to produce a reward function $\hat{r}_{\hat{\psi}}$.  This is followed by a policy optimization step, which, given a language model $\hat{\pi}_\theta$ parameterized by a vector $\theta$, minimizes $\mathcal{L}_\mathrm{policy}(\cdot | \hat{r}_{\hat{\psi}})$ to produce the fine-tuned language model $\hat{\pi}_{\hat{\theta}}$.  This minimization is typically carried out via stochastic gradient descent starting with $\theta$ initialized to match the base policy $\hat{\pi}_\theta = \overline{\pi}$.

\subsection{Direct Preference Optimization}

Direct preference optimization (DPO) approximates RLPO while bypassing the reward modeling step \cite{rafailov2023dpo}.  The single-step process can be derived by first observing that if  $\pi$ minimizes $\mathcal{L}_\mathrm{policy}(\cdot | r)$ then
\begin{equation}
\label{eq:reward-to-policy}
P_\pi(x) = \frac{P_{\overline{\pi}}(x) e^{r(x) |\data|/\beta}}{\sum_{x' \in \messages} P_{\overline{\pi}}(x') e^{r(x') |\data|/ \beta}},
\end{equation}
and therefore, 
\begin{equation}
\label{eq:r_in_terms_of_pi}
r(x) = \frac{\beta}{|\data|} \ln \frac{P_\pi(x)}{P_{\overline{\pi}}(x)} + \frac{\beta}{|\data|} \ln \sum_{x' \in \messages} P_{\overline{\pi}}(x') e^{r(x') |\data|/\beta}.
\end{equation}
Substituting $r$ in \eqref{eq:reward-loss} with the right-hand-side of \eqref{eq:r_in_terms_of_pi} yields a new loss function
\begin{equation}
\label{eq:dpo-loss}
\mathcal{L}_\mathrm{DPO}(\pi) = - \sum_{(\Yc, y) \in \data} \ln \frac{
\left(P_{\pi}(y)/ P_{\overline{\pi}}(y)\right)^{\beta/|\data|}}{\sum_{y'\in\Yc} 
\left(P_{\pi}(y') / P_{\overline{\pi}}(y')\right)^{\beta/|\data|}}.
\end{equation}
For a parameterized language model $\hat{\pi}_\theta$, minimizing $\mathcal{L}_\mathrm{DPO}$ produces a fine-tuned language model $\hat{\pi}_{\hat{\theta}}$.  

As we will see, RLPO and DPO both can give rise to egregious behavior when queries include more than two alternatives.  Recent work proposes a family of RLHF algorithms that unifies and extends RLPO and DPO \cite{azar2023general}.  Other so-called $\Psi$PO algorithms suffer in the same manner when queries include more than two alternatives.

\subsection{Inclusive Learning and Sequence Likelihood Calibration}
Unlike RLPO or DPO which maximize reward, inclusive learning (IL) aims to produce language models that reflect the diversity of preferences across the population \cite{xu2023shattering}.
IL produces a language model $\pi$ that simultaneously serves as a reward model.  In particular, reward is taken to be the log-probability $\ln P_\pi(x)$ assigned to a message.  Minimizing cross-entropy loss while regularizing toward the base policy $\overline{\pi}$ gives rise to an inclusive loss function
\begin{equation*}
\mathcal{L}_\mathrm{IL}(\pi) = - \sum_{(\Yc, y) \in \data} \ln \frac{P_{\pi}(y)}{\sum_{y'\in\Yc} P_{\pi}(y')} + \beta \KL(P_\pi\|P_{\overline{\pi}}).
\end{equation*}
For a parameterized language model $\hat{\pi}_\theta$, minimizing $\mathcal{L}_\mathrm{IL}$ produces a fine-tuned language model $\hat{\pi}_{\hat{\theta}}$.

While this IL algorithm was introduced in \cite{xu2023shattering}, it is closely related to other approaches proposed in the literature \cite{zhao2023slichf, hejna2023contrastive}, which share in the merits and faults of IL that we will discuss.  A notable representative is sequence likelihood calibration (SLiC) with human feedback \citep{zhao2023slichf}. For $|\Yc| = 2$, the loss function for SLiC is defined as
\begin{equation*}
    \Lc_\mathrm{SLiC}(\pi) = \sum_{(\Yc,y)\in\data} \left(\ln \frac{P_{\pi}(\Yc\setminus \{y\})}{P_{\pi}(y)} + \delta\right)_+ + \beta \KL (P_{\pi}\| P_{\overline{\pi}}),
\end{equation*}
for some scalar margin $\delta > 0$.  

\section{Perverse Incentives}
\label{se:perverse-incentives}

We will explain how RLPO, DPO, IL, and SLiC can fail to produce desirable results due to the fact that underlying models assume IIA.  Our explanations build on theoretical and computational results.  These results assume simple processes for generation of preference data, which we now describe.  In the cases of RLPO and DPO, failure modes arise only when queries include more than two alternative messages.

\subsection{Dichotomy Data}
\label{se:dichotomy-data}

Our simulation and theoretical results assume a particular data generating process, articulated by the following assumption.
\begin{restatable}[dichotomy data]{assumption}{AssumptionDichotomy}
\label{as:dichotomy}
Elements of $\data$ are sampled iid.  For each datum $(\Yc, Y) \in \data$, the choice $Y$ is sampled according to a dichotomy model $(\Zc, p_*, r_*)$.  For any $(\Yc,Y), (\Yc', Y') \in \data$, $|\Yc| = |\Yc'|$.  Each element of each set $\Yc$ of alternatives is sampled independently by a language model $\overline{\pi}$.  For any message category $\tau \in \{1,2\}$ and messages $y,y' \in \messages_\tau$, $P_{\overline{\pi}}(y) = P_{\overline{\pi}}(y')$.
\end{restatable}
Recall that, for the dichotomy model, there are two individual types, $\types = \{1,2\}$, and $r_*(x|z) = \1_{\messages_z}(x)$.  Under our assumption, each tuple $(\Yc, y)$ can be viewed as generated as follows.  First, each element of $\Yc$ is generated by sampling $\tau$ and then a message uniformly from $\messages_\tau$.  The distribution of $\tau$ is implied by the language model $\overline{\pi}$.  Then, $Z$ is sampled from $p_*$ and $Y$ is sampled uniformly from $\Yc \cap \messages_Z$.

\subsection{Architectures}
\label{se:architectures}

Each RLHF algorithm we have described operates by minimizing one or more loss functions.  The argument of each is itself a function.  Optimization is carried out by tuning parameters of an approximation architecture taking the form, for example, of a neural network.  Some of our results pertain to particular simple architectures chosen to produce transparent analyses specifically for data satisfying Assumption \ref{as:dichotomy}.  Our next two assumptions describe these architectures.  The first pertains to the reward function architecture.
\begin{restatable}[reward architecture]{assumption}{AssumptionDichotomyRewardArchitecture}
\label{as:dichotomy-reward-architecture}
For each $\psi \in \Re^2$, $z\in \Zc$, and $x\in \messages_z$, $\hat{r}_\psi(x) = \psi_z$.
\end{restatable}
Under this assumption, each reward function $r_\psi$ is parameterized by two scalars -- $\psi_1$ and $\psi_2$ -- which express estimates of rewards enjoyed by individuals who receive their desired type of message.

Our next assumption pertains to the policy architecture.
\begin{restatable}[policy architecture]{assumption}{AssumptionDichotomyPolicyArchitecture}
\label{as:dichotomy-policy-architecture}
For each $\theta \in \Re^2$, $z \in \Zc$, and $x \in \messages_z$, $P_{\hat{\pi}_\theta}(x) = e^{\theta_z} / (|\messages_1|e^{\theta_1} + |\messages_2|e^{\theta_2})$.
\end{restatable}
Under this assumption, each policy is identified by two scalars -- $\theta_1$ and $\theta_2$.  Increasing either $\theta_1$ or $\theta_2$ increases the chances of generating messages of the corresponding type.

\subsection{Simulation Setup}
\label{se:simulation}

Our simulations are carried out with data and architectures that satisfy Assumptions \ref{as:dichotomy}, \ref{as:dichotomy-reward-architecture}, and \ref{as:dichotomy-policy-architecture}.  Message sets are of cardinality $|\messages_1| = 10$ and $|\messages_2| = 100$ unless noted otherwise.  The choice model type distribution and the reward function are given by $p_*(1) = 0.6$, $p_*(2) = 0.4$, and $r_*(x|z) = \1_{\messages_z}(x)$.  The base language model $\overline{\pi}$ satisfies $P_{\overline{\pi}}(x) = 0.8 / |\messages_1|$ for $x \in \messages_1$ and $P_{\overline{\pi}}(x) = 0.2 / |\messages_2|$ for $x \in \messages_2$.  Hence, a dominant fraction $p_*(1) > p_*(2)$ of the population is of type $1$, individuals of that type prefer messages in $\messages_1$, and the baseline $\overline{\pi}$ tends to generate type $1$ messages an even larger fraction of the time than $p_*(1)$.  We use a regularization penalty coefficient of $\beta=1$.  Under these circumstances, it is surprising that, as we will see, RLHF algorithms can produce language models that almost always generate messages of type $2$.

\subsection{RLPO and DPO}
\label{se:rlpo-failure}

RLPO and DPO are designed to produce language models that generate desired messages.  As such, we should expect language models produced by these RLHF algorithms to gravitate toward messages in $\messages_1$, which are preferred by $60\%$ of the population.  As can be seen in Figure \ref{fig:rlpo-and-dpo-fail}, 
this is indeed the case when each choice set $|\Yc|$ contains two messages.  In particular, for sufficiently large datasets, language models produced by RLPO and DPO consistently generate elements of $\messages_1$.  However, with larger choice sets, the language models consistently generate elements of $\messages_2$.\footnote{Larger choice sets are common, for example, in applications where a language model is used to suggest alternative messages for use by a human agent who is assisting a user.  In such contexts, the human agent selects one of the alternatives or manually crafts a response.  The human agent's choice serves as feedback that can be used to train the language model in order to improve subsequent suggestions.}

\begin{figure}[htb]
\begin{center}
\includegraphics[scale=0.28]{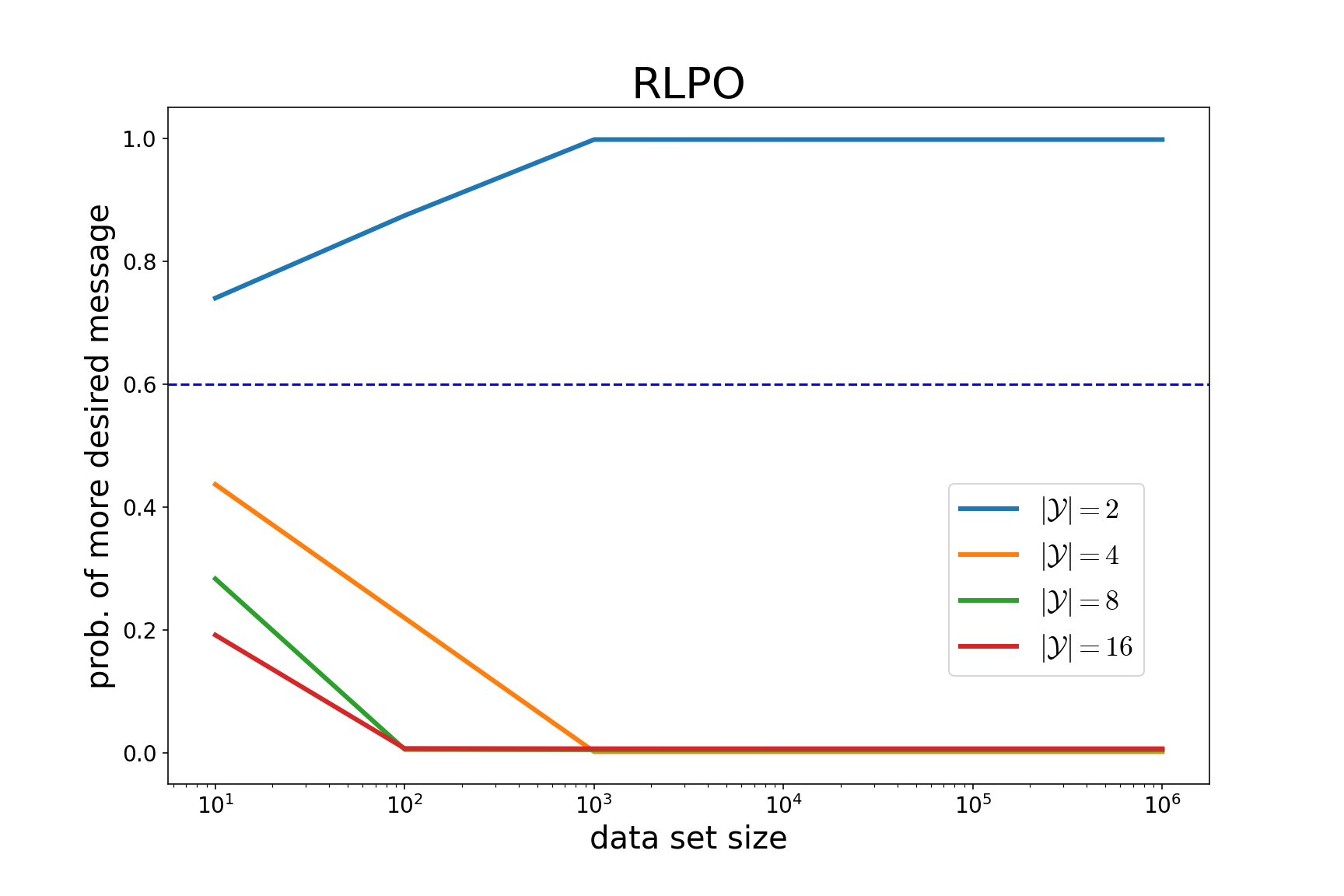} \hspace{-0.35in}
\includegraphics[scale=0.28]{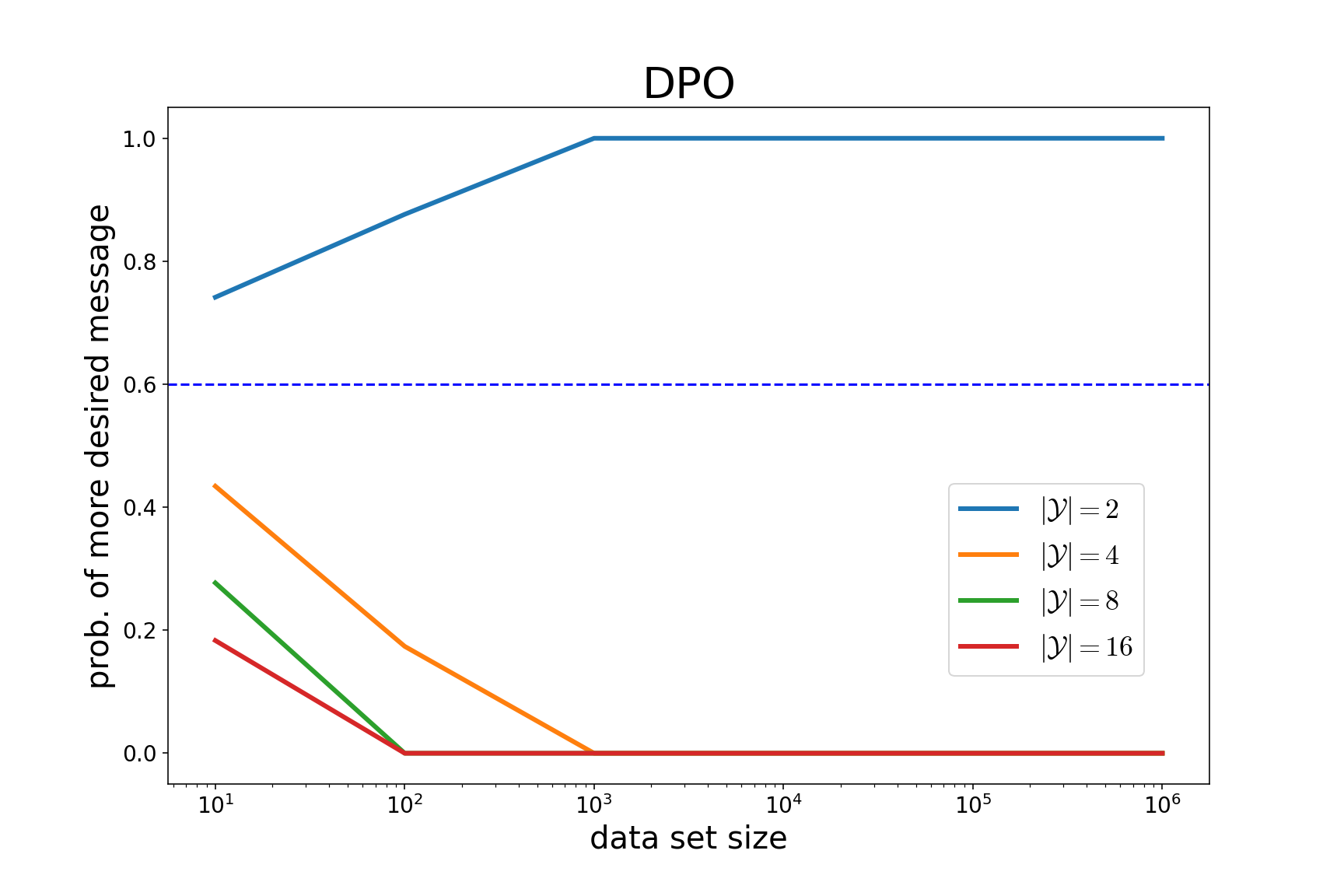}
\caption{For choice sets $\Yc$ containing two messages, as $\data$ grows, RLPO and DPO produce language models that consistently generate messages most likely to be preferred.  However, with larger choice sets, less preferred messages are consistently generated.  Each plot is averaged over one hundred independent simulations.}
\label{fig:rlpo-and-dpo-fail}
\end{center}
\end{figure}

It may seem surprising that RLPO and DPO fail so egregiously when choice sets include more than a pair of alternatives, and it is natural to wonder whether this could be due to technical details of our simulation.  However, the following theoretical results establish in greater generality that minimizing RLPO loss functions can lead to egregious outcomes.


\begin{restatable}[RLPO failure]{proposition}{dichotomyRlpo}
\label{prop:dichotomy-rlpo-policy}
Under Assumptions \ref{as:dichotomy}, \ref{as:dichotomy-reward-architecture}, and \ref{as:dichotomy-policy-architecture}, for all $|\Yc|\ge 3$, if $p_*(1) < F(P_{\overline{\pi}}(\messages_1))$ with $F(\zeta) = \frac{\zeta - \zeta^{|\Yc|}}{1 - \zeta^{|\Yc|} - (1-\zeta)^{|\Yc|}}$, then, as $|\data| \rightarrow \infty$,
\begin{equation}
P_{\hat{\pi}_{\hat{\theta}}}(\Mc_1) \stackrel{p.}{\longrightarrow} 0 \qquad \text{and} \qquad \nn P_{\hat{\pi}_{\hat{\theta}}}(\Mc_2) \stackrel{p.}{\longrightarrow} 1,
\end{equation}
where $\hat{\theta}$ minimizes $\Lc_\mathrm{policy}(\hat{\pi}_\theta|\hat{r}_{\hat{\psi}})$, and $\hat{\psi} \in \argmin_\psi \Lc_\mathrm{reward}(\hat{r}_\psi)$ if the loss $\Lc_\mathrm{reward}(\hat{r}_\psi)$ has a minimizer, and $\hat{\psi}=0$ otherwise.
\end{restatable}


Given that DPO is designed to approximate RLPO, one would expect a similar theoretical result to hold for DPO.  The following proposition formally establishes this.

\begin{restatable}[DPO failure]{proposition}{dichotomyDpo}
\label{prop:dichotomy-dpo}
Under Assumptions \ref{as:dichotomy}, \ref{as:dichotomy-reward-architecture}, and \ref{as:dichotomy-policy-architecture}, for all $|\Yc|\ge 3$, if $p_*(1) < F(P_{\overline{\pi}}(\messages_1))$ with $F(\zeta) = \frac{\zeta - \zeta^{|\Yc|}}{1 - \zeta^{|\Yc|} - (1-\zeta)^{|\Yc|}}$, then as $|\data| \rightarrow \infty$,
\begin{equation}
P_{\hat{\pi}_{\hat{\theta}}}(\Mc_1) \stackrel{p.}{\longrightarrow} 0 \qquad \text{and} \qquad \nn P_{\hat{\pi}_{\hat{\theta}}}(\Mc_2) \stackrel{p.}{\longrightarrow} 1,
\end{equation}
where $\hat{\theta} \in \argmin_{\theta} \Lc_\mathrm{DPO}(\hat{\pi}_{\theta})$ if the loss $\Lc_\mathrm{DPO}(\hat{\pi}_{\theta})$ has a minimizer, and $\hat{\theta} = 0$ otherwise.
\end{restatable}


To understand what causes these failures, let us consider as a thought experiment a simplified data generating process where each choice set contains the same triple $\Yc = \{y_1,y_2,y_3\}$, with $y_1,y_2 \in \messages_1$ and $y_3 \in \messages_2$.  Since $p_*(1) = 0.6$ and $p_*(2) = 0.4$, choice probabilities generating the preference data are given by
$\Pr(Y = 1| \Yc) = \Pr(Y = 2| \Yc) = 0.3$ and $\Pr(Y = 3 | \Yc) = 0.4$.  As the dataset grows, minimizing $\Lc_\mathrm{reward}(\hat{r}_\psi)$ identifies parameters $\hat{\psi}$ to match these probabilities, if possible.  In particular,
\begin{equation}
\label{eq:IIA-RLPO-example}
\frac{e^{\hat{r}_{\hat{\psi}}(y_1)}}{\sum_{i=1}^3 e^{\hat{r}_{\hat{\psi}}(y_i)}} = \frac{e^{\hat{r}_{\hat{\psi}}(y_2)}}{\sum_{i=1}^3 e^{\hat{r}_{\hat{\psi}}(y_i)}} = 0.3 \qquad \text{and} \qquad  \frac{e^{\hat{r}_{\hat{\psi}}(y_3)}}{\sum_{i=1}^3 e^{\hat{r}_{\hat{\psi}}(y_i)}} = 0.4.
\end{equation}
These equations imply that $\hat{\psi}_2 - \hat{\psi}_1 = \ln (4/3) > 0$.  Hence, the estimated reward $\hat{r}_{\hat{\psi}}(x)$ is maximized, as RLPO aims to do, by generating messages $x \in \messages_2$.  DPO is designed to approximate RLPO and therefore leads to similar behavior.

In our thought experiment, the fact that each choice set had twice as many more desirable ($\messages_1$) than the less desirable ($\messages_2$) messages biased reward estimates in favor of $\messages_2$.  This tendency is expressed in the above propositions though the requirement that $F(P_{\overline{\pi}}(\messages_1)) > p_*(1)$, with $F(\cdot)$ defined in Propositions \ref{prop:dichotomy-rlpo-policy} and \ref{prop:dichotomy-dpo}.  The baseline policy $\overline{\pi}$ is used to generate choice sets, and perhaps surprisingly, the fact that it biases samples toward desirable messages leads to undesirable outcomes.

The root cause is that $\hat{r}_{\hat{\psi}}$ assumes IIA while the process generating preferences does not.  In our thought experiment, for example, where $\Pr(Y = 1| \Yc) = \Pr(Y = 2| \Yc) = 0.3$ and $\Pr(Y = 3 | \Yc) = 0.4$, if the choice sets were instead to only contain two alternatives $\Yc = \{y_2,y_3\}$ of different types, then we would have $\Pr(Y = 2| \Yc) = 0.6$ and $\Pr(Y = 3 | \Yc) = 0.4$, giving rise to different equations
\begin{equation}
\label{eq:IIA-RLPO-example-pairwise}
\frac{e^{\hat{r}_{\hat{\psi}}(y_2)}}{\sum_{i=2}^3 e^{\hat{r}_{\hat{\psi}}(y_i)}} = 0.6 \qquad \text{and} \qquad  \frac{e^{\hat{r}_{\hat{\psi}}(y_3)}}{\sum_{i=2}^3 e^{\hat{r}_{\hat{\psi}}(y_i)}} = 0.4.
\end{equation}
These new equations imply that $\hat{\psi}_2 - \hat{\psi}_1 = \ln (2/3) < 0$, correctly identifying $\messages_1$ as more desirable than $\messages_2$.  The irrelevant alternative $y_1$ is able to distort estimates because the reward model assumes IIA.

\subsection{IL and SLiC}
\label{se:il-slic}

IL is designed to produce language models that reflect the diversity of preferences across the population \cite{xu2023shattering}.  In particular, IL ought to generate messages in $\messages_1$ and $\messages_2$ according to the probabilities $0.6$ and $0.4$ with which they are preferred by random individuals.  As can be seen in Figure \ref{fig:inclusive-fails}, this is indeed the case when $|\Yc| = 2$ and $\messages_1 = \messages_2 = 10$.  However, as the set $\messages_2$ of less desired messages grows, the probability $P_{\hat{\pi}_{\hat{\theta}}}(\messages_1)$ of generating a more desired message vanishes.

\begin{figure}[htb]
\begin{center}
\includegraphics[scale=0.3]{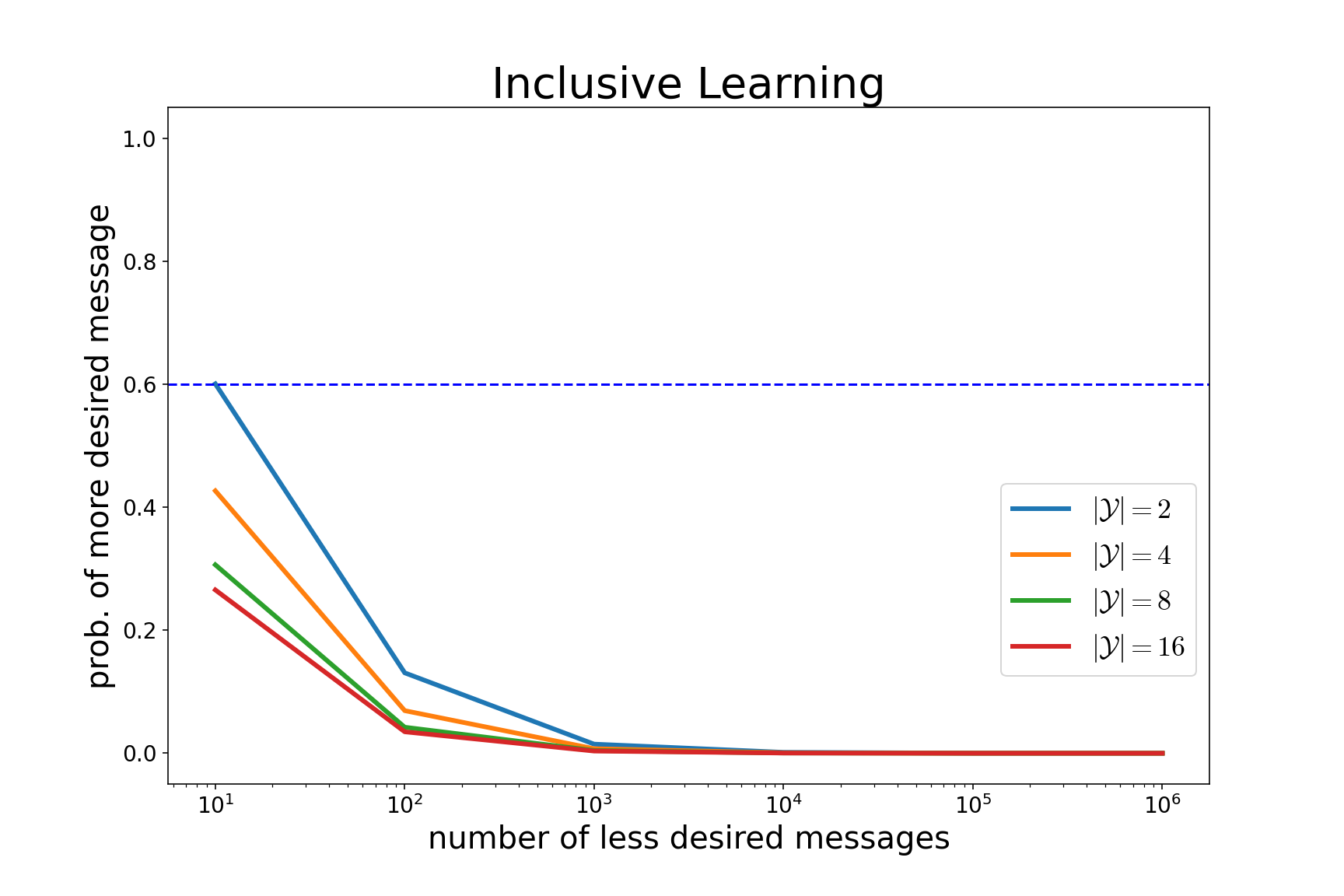}
\caption{Inclusive learning tends to generate a message in the less desired set $\messages_2$ as that set grows.}
\label{fig:inclusive-fails}
\end{center}
\end{figure}

The following proposition formalizes this phenomenon, establishing that IL fails as $\data$ and $\messages_2$ grow, not just for the specific simulated examples, but more generally.
\begin{restatable}[IL failure]{proposition}{dichotomyInclusive}
\label{prop:dichotomy-inclusive}
Under Assumptions \ref{as:dichotomy} and \ref{as:dichotomy-policy-architecture}, if $|\Yc| \ge 2$ and $\hat{\theta}$ minimizes $\Lc_\mathrm{IL}(\hat{\pi}_\theta)$, then for fixed $\Mc_1$, 
\begin{equation}
\lim_{|\data|\to\infty} \Pr\left(\lim_{|\messages_2| \to\infty} P_{\hat{\pi}_{\hat\theta}}(\messages_1) = 0\right) = 1.
\end{equation}
\end{restatable}

To understand what causes these failures, let us consider another thought experiment.  Suppose that each choice set contains two messages $\Yc = \{y_1,y_2\}$, with $y_1 \in \messages_1$ and $y_2 \in \messages_2$.  Choice probabilities generating the preference data are given by
$\Pr(Y = 1| \Yc) = p_*(1) = 0.6$ and $\Pr(Y = 2 | \Yc) = p_*(2) = 0.4$.  As the dataset grows, minimizing $\Lc_\mathrm{IL}(\hat{\pi}_\theta)$ identifies parameters $\hat{\theta}$ to match these probabilities, if possible.  In particular, 
\begin{equation}
\label{eq:IIA-IL-example}
\frac{P_{\hat{\pi}_{\hat{\theta}}}(y_1)}{P_{\hat{\pi}_{\hat{\theta}}}(y_1) + P_{\hat{\pi}_{\hat{\theta}}}(y_2)} = 0.6 \qquad \text{and} \qquad  \frac{P_{\hat{\pi}_{\hat{\theta}}}(y_2)}{P_{\hat{\pi}_{\hat{\theta}}}(y_1) + P_{\hat{\pi}_{\hat{\theta}}}(y_2)} = 0.4.
\end{equation}
This implies that $P_{\hat{\pi}_{\hat{\theta}}}(y_1) / P_{\hat{\pi}_{\hat{\theta}}}(y_2) = 3/2$.  Since $P_{\hat{\pi}_{\hat{\theta}}}(y_i) = P_{\hat{\pi}_{\hat{\theta}}}(\messages_i) / |\messages_i|$ for $i \in \{1,2\}$, we have
\begin{equation}
\label{eq:IIA-IL-example-math}
\frac{3}{2} = \frac{P_{\hat{\pi}_{\hat{\theta}}}(\messages_1) / |\messages_1|}{P_{\hat{\pi}_{\hat{\theta}}}(\messages_2) / |\messages_2|} = \frac{|\messages_2|}{10} \cdot \frac{P_{\hat{\pi}_{\hat{\theta}}}(\messages_1)}{1 - P_{\hat{\pi}_{\hat{\theta}}}(\messages_1)}.
\end{equation}
Hence, $P_{\hat{\pi}_{\hat{\theta}}}(\messages_1) = 15 / (15 + |\messages_2|)$.  It follows that, as $\messages_2$ grows, $P_{\hat{\pi}_{\hat{\theta}}}(\messages_1)$ vanishes.

While the number of messages in $\messages_2$ does not influence how a human would choose between two messages, (\ref{eq:IIA-IL-example-math}) implies that it impacts the probabilities $P_{\hat{\pi}_{\hat{\theta}}}(\messages_1)$ and $P_{\hat{\pi}_{\hat{\theta}}}(\messages_2)$ of generating more or less liked messages.  The more equivalent messages there are in $\messages_2$, the more likely the resulting fine-tuned language model is to produce less-liked messages.  This is again due to the fact that the model underlying IL satisfies IIA while the preference data generating process does not.

A similar reasoning implies that SLiC also experiences the same type of failure, the proof of which we defer to the appendix.
\begin{restatable}[SLiC failure]{proposition}{dichotomySlic}
\label{prop:dichotomy-slic}
Under Assumptions \ref{as:dichotomy} and \ref{as:dichotomy-policy-architecture}, if $|\Yc| = 2$ and $\hat{\theta}$ minimizes $\Lc_\mathrm{SLiC}(\hat{\pi}_\theta)$, for fixed $\Mc_1$, 
\begin{equation}
\lim_{|\data|\to\infty} \Pr\left(\lim_{|\messages_2| \to\infty} P_{\hat{\pi}_{\hat\theta}}(\messages_1) = 0\right) = 1.
\end{equation}
\end{restatable}

This observation suggests that using IL or SLiC to fine-tune real language models will induce a bias toward generating longer messages.  To see why, suppose there are two ideas that might be expressed as responses to a prompt, and that the first requires a about ten words, while the second requires about one thousand words.
Suppose that $60\%$ of the population would prefer the first idea while $40\%$ would prefer the second.
Because the number of ways to express information scales quickly with the length of message required to express that information, there are likely to be many more roughly equivalent expressions of the second idea than the first.  Since IL and SLiC biases generation toward messages with many equivalent alternatives, the fine-tuned language model would tend to express the second idea, even though more people prefer the first idea.

\section{Empirical Study}
\label{se:empirical}

We next demonstrate that the sort of egregious behavior we have identified manifests when learning reward models of practical scale from realistic datasets.  In particular, we fit reward models that build on the PaLM 2 \citep{palm2} XXS language model to data generated using GPT-3.5 and GPT-4 \citep{openai2023gpt4}.  Each query includes reponses generated by GPT-3.5 and GPT-4.  We additionally use GPT-4 to simulate choices made by human annotators.

Our results establish that, when training queries each include a pair of responses, the reward model correctly learns to assign higher scores to GPT-4 responses.  However, when training queries each include four responses, the reward model erroneously assigns favorable scores to GPT-3.5 responses.  It is striking that a seemingly innocuous change to the training query format gives rise to such egregious behavior.  The results are summarized in Figure \ref{fig:empirical}.   Before discussing them in detail, we will describe the datasets, the reward model architecture, and the training algorithm.

\begin{figure}
\centering
\begin{subfigure}[h]{0.45\textwidth}
    \centering
    \includegraphics[width=0.75\textwidth]{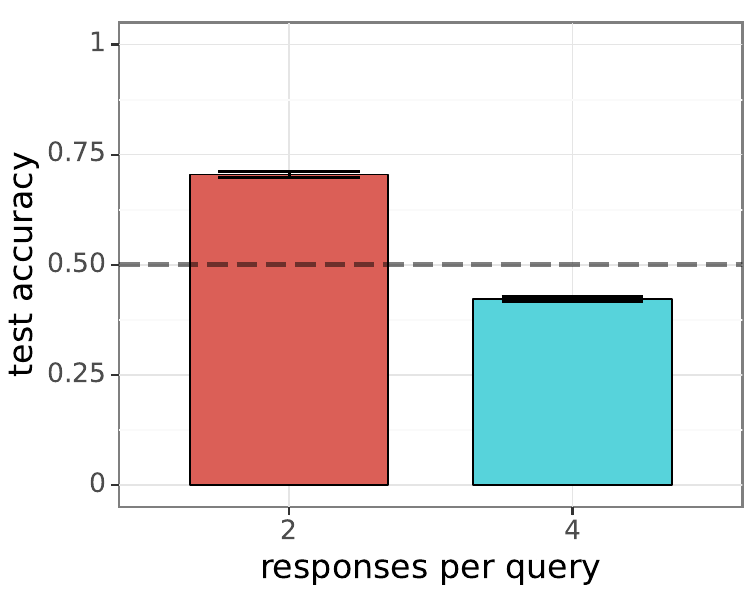}
    \caption{Test accuracy of a standard reward model trained on preference data with 2 or 4 responses per query. The dashed line indicates a purely random baseline.}
    \label{fig:empirical-accuracy}
\end{subfigure}
\hspace{1cm}
\begin{subfigure}[h]{0.45\textwidth}
    \centering
    \includegraphics[width=0.75\textwidth]{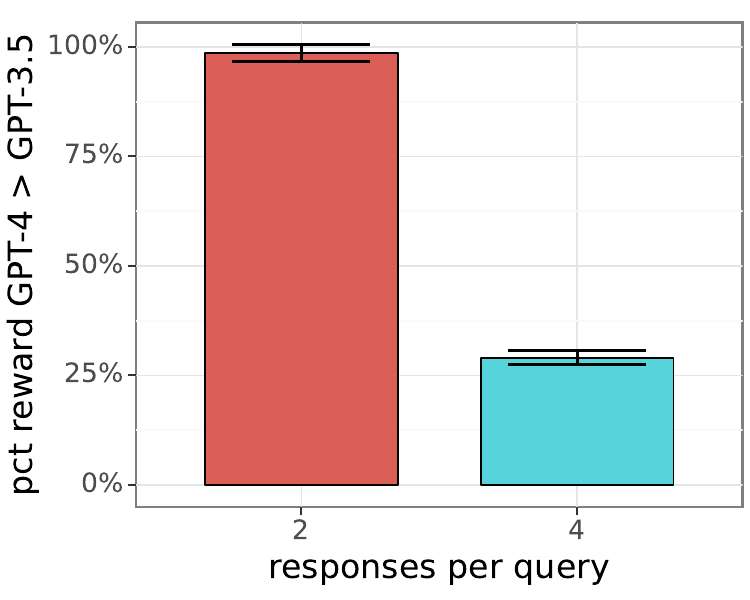}
    \caption{Percentage of test data where the reward model assigns a higher score to GPT-4's response.}
    \label{fig:empirical-preference}
\end{subfigure}
\caption{Standard reward model training can lead to egregious outcomes when training data involves more than two responses per query.}
\label{fig:empirical}
\end{figure}

\subsection{Preference Datasets}

For simplicity, we consider only a single prompt.  While other prompts lead to similar results, we arbitrarily choose to present results produced by the following prompt, based on a popular movie this year:
\begin{displayquote}
\begin{tabular}{ll}
{\bf prompt} & {\it Did Oppenheimer win a Nobel Prize?}
\end{tabular}
\end{displayquote}

We sample candidate responses using GPT-3.5 and GPT-4 with the temperature parameter of each language model set to the default value of one.  Here are representative responses:
\begin{displayquote}
\begin{tabular}{ll}
{\bf GPT-3.5 response} & {\it No, Oppenheimer did not win the Nobel Prize.} \\
{\bf GPT-4 response} & {\it No, Robert Oppenheimer, often called the ``father of the atomic bomb''} \\
& {\it for his role in the Manhattan Project, did not win a Nobel Prize.}
\end{tabular}
\end{displayquote}
GPT-3.5 tends to produce concise responses relative to the more informative ones from GPT-4.  From each of these two language models, we sample 100 responses for training and validation, and 100 responses for testing. We subsample from these responses when constructing each training or validation query, or each test sample.  We produce two training sets that differ in the number of alternative responses per query.  The first training set has 2 responses per query, one from GPT-3.5 and one from GPT-4.  The second training set has 4 responses per query, one from GPT-3.5 and three from GPT-4. We generate 800 queries for each training set. Additionally, we generate 200 queries for each validation set.  These are used in tuning hyperparameters to optimize validation accuracy.  We produce a single test dataset made up of 1000 queries, each with one response from GPT-3.5 and one from GPT-4.

To simulate human annotator choices, we prompt GPT-4 to select its favorite response among a set.  We use two types of prompts, which express preference for concise and informative responses respectively:
\begin{displayquote}
\begin{tabular}{ll}
{\bf concise choice prompt} & {\it Suppose that you are looking for a concise answer to the following} \\
& {\it  question.  Which response below do you like the best...} \\
{\bf informative choice prompt} & {\it Suppose that you are trying to learn more about the following question.} \\
& {\it Which response below do you like the best...}
\end{tabular}
\end{displayquote}
We find that prompting for a concise choice typically favors GPT-3.5, while prompting for an informative choice typically favors GPT-4.  When simulating a human annotator, we sample one of these two types of choice prompts according to probabilities 0.3 and 0.7, respectively.  Using these choice prompts, we find that regardless of the number of responses per query, our simulated human annotators select responses from GPT-4 approximately 70\% of the time.

\subsection{Reward Model Architecture and Training Algorithm}
Our reward model builds on the PaLM 2 \citep{palm2} XXS language model.  In particular, given a context formed by concatenating a prompt and a response, we take the final embedding produced by the base language model and apply a linear layer to obtain reward. We train the linear layer, as well as the base language model, to minimize cross-entropy loss on preference data. For each training dataset, we train the reward model over 150 gradient steps using the Adafactor optimizer \citep{shazeer2018adafactor} with a learning rate of 1e-4 and batch size of 16. The hyperparameters are tuned to optimize validation accuracy. We average results over 3 seeds.

\subsection{Results}
Figure \ref{fig:empirical-accuracy} plots the test accuracy of the reward model trained on queries with 2 or 4 responses. When each training query includes only a pair of responses, the reward model correctly selects the preferred response for around 70\% of test samples.  However, with four responses per training query, the test accuracy drops below 50\%.  In other words, choice predictions generated by the reward model fare worse than random guessing.  Figure \ref{fig:empirical-preference} provides more insight into how learned reward models score responses.  Since the majority of simulated annotators favor GPT-4, we would expect that the learned reward model assigns higher scores to GPT-4 responses.  We see in Figure \ref{fig:empirical-preference} that, when trained on pairwise comparisons, the reward model almost always assigns higher scores to GPT-4 responses. However, when trained on queries each with 4 responses, the reward model tends to erroneously assign higher scores to GPT-3.5 responses.

These results highlight a failing of the standard reward learning approach.  The culprit is the IIA assumption made in reward learning, which is violated by the data as we will now explain.  Intuitively, each simulated human annotator prefers either concise or informative responses but is relatively indifferent between alternatives in each of these two categories.  Consider a simulated annotator who, when presented with one concise and one informative response, selects the former with probability 2/5.  Under the IIA assumption, increasing the number of informative alternatives from one to three would increase the probability of choosing the concise response to 2/3.  However, if the annotator is indifferent between the three informative responses, the probability of choosing the concise response should remain 2/5, regardless of whether there are one or three informative alternatives.


\section{Closing Remarks}

The essential role of RLHF in the success of generative AI has attracted tremendous effort and resources to the subject.  It is worthwhile to think carefully about foundations of this methodology to inform algorithmic innovation and ultimately produce reliable AIs.  What we have presented, which points out how the IIA assumption adopted by models underlying current algorithms can give rise to egregious behavior, offers one of what will hopefully be many steps in this direction.


The design of RLHF algorithms that mitigate perverse incentives we have identified remains a subject for future research.  Trade-offs between query formats and how feedback is processed also deserve further study.  For example, while RLPO can fail when queries present more than two alternatives, one can heuristically convert a feedback from such a query to multiple pairwise choices.  This could offer more robust results, though there may be other flaws to this approach.

\appendix

\section{Proofs}

\subsection{RLPO}

In this subsection, we provide the proof for Proposition \ref{prop:dichotomy-rlpo-policy}.  We start by defining two key quantities for messages of category $1$ in $\data$.  First, we let
$$
\rho_\mathrm{chosen} = \frac{1}{|\data|} \sum_{(\Yc,y)\in\data} \1\left(y\in \messages_1\right)
$$
denote the fraction of data in $\data$ where the chosen message is in $\messages_1$.  Second, we let
$$
\rho_\mathrm{data} = \frac{1}{|\data|} \sum_{(\Yc,y)\in\data} \frac{|\Yc \cap \messages_1|}{|\Yc|}
$$
denote the fraction of messages in $\data$ that belong to $\messages_1$.  
\begin{lemma}
\label{le:strictly-convex-zero}
For all strictly convex functions $f:\Re \to \Re$ such that $f'(1) < 0$, if $f(\zeta) = 0$ for some $\zeta \in\Re$, then $\zeta > 1 - \frac{f(1)}{f'(1)}$.
\end{lemma}
\begin{proof}
Since $f$ is strictly convex, $f(\zeta) > f(1) + f'(1)(\zeta-1)$.  Since $f'(1) < 0$, this implies $\zeta - 1 > -\frac{f(1)}{f'(1)}$. 
\end{proof}

\begin{lemma}
\label{le:deterministic-proportion-of-preferred}
Under Assumption \ref{as:dichotomy-reward-architecture}, if $\rho_\mathrm{data} > \rho_\mathrm{chosen} > 0$, then the minimizer $\hat{\psi}$ of $\Lc_\mathrm{reward}(\hat{r}_\psi)$ exists and satisfies
$$
\hat{\psi}_2 - \hat{\psi}_1 > \frac{\rho_\mathrm{data} - \rho_\mathrm{chosen}}{1 + \rho_\mathrm{data} - \rho_\mathrm{chosen}}.
$$
\end{lemma}
\begin{proof}
Since $\Lc_\mathrm{reward}$ is invariant under translations of $\psi$, without loss of generality, set $\psi_1 = 0$.  
\begin{align*}
\Lc_\mathrm{reward}(\hat{r}_\psi) &= -\sum_{(\Yc,y)\in\data} \ln \frac{e^{\hat{r}_\psi(y)}}{\sum_{y'\in\Yc} e^{\hat{r}_\psi(y')}} = - \psi_2 |\data| (1-\rho_\mathrm{chosen}) + \sum_{(\Yc,y)\in\data} \ln\left(|\Yc\cap \messages_1| + |\Yc\cap\messages_2| e^{\psi_2}\right).
\end{align*}
Note that this loss function is strictly convex and thus a local minima is also a global minima.  Examining the derivative, we see that $\hat{\psi}_2$, if exists, satisfies
$$
\sum_{(\Yc,y)\in\data} \frac{|\Yc\cap\messages_1|}{|\Yc\cap\messages_1| + |\Yc\cap\messages_2| e^{\hat{\psi}_2}}  - |\data|\rho_\mathrm{chosen} = 0.
$$
With a change of variable, let $x = e^{\psi_2}$ and
$$f(x) = \sum_{(\Yc,y)\in\data} \frac{|\Yc\cap\messages_1|}{|\Yc\cap\messages_1| + |\Yc\cap\messages_2| x}  - |\data|\rho_\mathrm{chosen}.$$
Clearly, $f(x)$ is strictly decreasing on $\Re_+$ and $f(x) \to -|\data|\rho_\mathrm{chosen}$ as $x \to\infty$.  Since $\rho_\mathrm{chosen} > 0$, there exists a unique $\zeta > 0$ such that $f(\zeta) = 0$.  Thus, $\hat{\psi}_2 = \ln\zeta$ exists.  Further, we note that $f(x)$ is strictly convex and $f'(1) < 0$.  By Lemma \ref{le:strictly-convex-zero}, 
\begin{equation}
\zeta > 1 - \frac{f(1)}{f'(1)} = 1 + |\data|(\rho_\mathrm{data} - \rho_\mathrm{chosen}) \cdot |\Yc|^2 \left(\sum_{(\Yc,y)\in\data} |\Yc\cap\messages_1|\cdot |\Yc\cap\messages_2|\right)^{-1} > 1 + \rho_\mathrm{data} - \rho_\mathrm{chosen}. \label{eq:zeta-1}
\end{equation}
The result follows from combining Equation \eqref{eq:zeta-1} and the fact that $\ln(1+\alpha) \ge \frac{\alpha}{1+\alpha}$ for all $\alpha > -1$. 

\end{proof}

\noindent
For all $\eta>0$, define the high probability event
\begin{equation}
\Ec_\eta(\data) = \left\{ \rho_\mathrm{data} - \rho_\mathrm{chosen} > \eta \right\} \bigcap \left\{\rho_\mathrm{chosen} > 0\right\}.
\end{equation}

\begin{lemma}
\label{le:deterministic-failure}
Under Assumptions \ref{as:dichotomy-reward-architecture} and \ref{as:dichotomy-policy-architecture}, if there exists $\eta>0$ such that $\Pr(\Ec_\eta(\data)) \to 1$ as $|\data|\to\infty$, then as $|\data| \to \infty$,
$$
P_{\hat{\pi}_{\hat{\theta}}}(\Mc_1) \stackrel{p.}{\longrightarrow} 0 \qquad \text{and} \qquad \nn P_{\hat{\pi}_{\hat{\theta}}}(\Mc_2) \stackrel{p.}{\longrightarrow} 1,
$$
where $\hat{\theta}$ minimizes $\Lc_\mathrm{policy}(\hat{\pi}_\theta|\hat{r}_{\hat{\psi}})$, and $\hat{\psi} \in \argmin_\psi \Lc_\mathrm{reward}(\hat{r}_\psi)$ if the loss $\Lc_\mathrm{reward}(\hat{r}_\psi)$ has a minimizer, and $\hat{\psi}=0$ otherwise.
\end{lemma}
\begin{proof}
By Lemma \ref{le:deterministic-proportion-of-preferred}, under event $\Ec_\eta(\data)$, $\hat{\psi} = \argmin_\psi \Lc_\mathrm{reward}(\hat{r}_\psi)$ and $\hat{\psi}_2 - \hat{\psi}_1 > \frac{\rho_\mathrm{data} - \rho_\mathrm{chosen}}{1 + \rho_\mathrm{data} - \rho_\mathrm{chosen}} > \frac{\eta}{1+\eta}$.  When $\beta = 0$, $\hat{\pi}_{\hat{\theta}}$ maximizes reward.  Thus, under event $\Ec_\eta(\data)$, $P_{\hat{\pi}_{\hat{\theta}}}(\Mc_1) = 0$ and $P_{\hat{\pi}_{\hat{\theta}}}(\Mc_2) = 1$.  When $\beta>0$, $\hat{\pi}_{\hat{\theta}}$ satisfies
$$
P_{\hat{\pi}_{\hat{\theta}}}(y) = \frac{P_{\overline{\pi}}(y) e^{\hat{r}_{\hat\psi}(y) |\data|/\beta}}{\sum_{y' \in \Mc} P_{\overline{\pi}}(y') e^{\hat{r}_{\hat\psi}(y') |\data|/ \beta}}.
$$
Under event $\Ec_\eta(\data)$,
$$
\frac{P_{\hat{\pi}_{\hat{\theta}}}(\messages_1)}{P_{\hat{\pi}_{\hat{\theta}}}(\messages_2)} = \frac{p_1}{p_2} \cdot e^{\left(\hat{\psi}_1 - \hat{\psi}_2\right) |\data|/\beta} < e^{-\frac{\eta}{1+\eta}\cdot \frac{|\data|}{\beta}}.
$$
For all $\epsilon > 0$, there exists $N > 0$ such that $e^{-\frac{\eta}{1+\eta}\cdot \frac{|\data|}{\beta}} < \epsilon$ for all $|\data| > N$.  In turn, for all $\beta \ge 0$ and $\epsilon > 0$,  there exists $N > 0$ such that for all $|\data| > N$,
$$
\Pr\left(\frac{P_{\hat{\pi}_{\hat{\theta}}}(\messages_1)}{P_{\hat{\pi}_{\hat{\theta}}}(\messages_2)} \ge \epsilon \right) \le \Pr\left(\Ec_\eta(\data)^c\right).
$$
The proof follows from the fact that $\Pr(\Ec_\eta(\data)^c) \to 0$ as $|\data|\to\infty$ and $P_{\hat{\pi}_{\hat{\theta}}}(\messages_1) + P_{\hat{\pi}_{\hat{\theta}}}(\messages_2) = 1$. 

\end{proof}

\noindent
Finally, we prove that under the dichotomy data assumption, there exists $\eta>0$ such that the event $\Ec_\eta(\data)$ holds with high probability.  Define 
\[
\Ic = \big \{(\Yc, y) \in\data \, | \, \Yc \cap \messages_1 \ne \emptyset, \, \Yc\cap\messages_2\ne\emptyset \big \}
\]
as the subset of $\data$ that contains messages in both categories.  The first lemma below shows that most of the data contain both categories of messages with high probability. 

\begin{lemma}
\label{le:high-probability-both-types}
Under Assumption \ref{as:dichotomy}, if $P_{\overline{\pi}}(\messages_1) \in (0,1)$, $|\Yc| \ge 2$, then there exists a constant $\gamma = \gamma(|\Yc|) \in(0,1)$ such that for all $|\data|>0$, $\Pr(|\Ic| \le |\data|(1-\gamma)) \le e^{-\frac{|\data| (1-\gamma)\gamma^2}{8}}$.
\end{lemma}
\begin{proof}
Under Assumption \ref{as:dichotomy}, for each $(\Yc, y)\in\data$, the membership $\tau$ of each message $y'\in\Yc$ follows $\mathrm{Bernoulli}(P_{\overline{\pi}}(\messages_1))$.  Let $s = 1 - P_{\overline{\pi}}(\messages_1)^{|\Yc|} - P_{\overline{\pi}}(\messages_2)^{|\Yc|} < 1$.  Then $s$ equals the probability of sampling both categories of messages in each datum, and $|\Ic| \sim \mathrm{Binomial}(|\data|, s)$.  By Chernoff's bound,
\begin{equation}
\Pr\left(|\Ic| \le s\cdot|\data|s\right) \le e^{-\frac{|\data|s(1-s)^2}{2}}. \nn
\end{equation}
Letting $\gamma = 1 - s^2 \in (0,1)$, we have $s(1-s)^2 > (1-\gamma)\gamma^2/4$.  Therefore, 
\begin{equation}
\Pr\left(|\Ic| \le |\data|(1-\gamma)\right) \le e^{-\frac{|\data| (1-\gamma)\gamma^2}{8}}.
\end{equation}
\end{proof}

\begin{lemma}
\label{le:high-probability-event}
Under Assumption \ref{as:dichotomy}, for all $|\Yc|\ge 3$, if $p_*(1) < F(P_{\overline{\pi}}(\messages_1))$ with $F(\zeta) = \frac{\zeta - \zeta^{|\Yc|}}{1 - \zeta^{|\Yc|} - (1-\zeta)^{|\Yc|}}$, then there exists $\eta>0$ such that $\Pr(\Ec_\eta(\data)) \to 1$ as $|\data|\to\infty$.
\end{lemma}
\begin{proof}
Let $\delta = F(P_{\overline{\pi}}(\messages_1)) - p_*(1)>0$.  Under Assumption \ref{as:dichotomy}, $|\Yc\cap\messages_1| \sim \mathrm{Binomial}(|\Yc|,P_{\overline{\pi}}(\messages_1))$, giving
$$
\E\left[\frac{|\Yc \cap \messages_1|}{|\Yc|} \mid (\Yc,y)\in\Ic\right] = \frac{P_{\overline{\pi}}(\messages_1) - P_{\overline{\pi}}(\messages_1)^{|\Yc|}}{1 - P_{\overline{\pi}}(\messages_1)^{|\Yc|} - P_{\overline{\pi}}(\messages_2)^{|\Yc|}} = F(P_{\overline{\pi}}(\messages_1)).
$$
On the other hand, 
$$
\E\left[\1(y\in\messages_1) \mid (\Yc,y)\in\Ic \right] = p_*(1).
$$
For $(\Yc,y)\notin\Ic$, observe that $\frac{|\Yc \cap \messages_1|}{|\Yc|} = \1(y\in\messages_1)$.  Thus 
\begin{equation}
\label{eq:difference}
\frac{1}{|\Ic|}\sum_{(\Yc,y)\in\Ic} \frac{|\Yc \cap \messages_1|}{|\Yc|} - \frac{1}{|\Ic|}\sum_{(\Yc,y)\in\Ic} \1(y\in\messages_1) = \frac{|\data|}{|\Ic|}(\rho_\mathrm{data} - \rho_\mathrm{chosen}).
\end{equation}
By Hoeffding's inequality, 
\begin{align*}
\Pr\left( \frac{1}{|\Ic|}\sum_{(\Yc,y)\in\Ic} \frac{|\Yc \cap \messages_1|}{|\Yc|} - F(P_{\overline{\pi}}(\messages_1)) \le - \frac{\delta}{2} \mid \Ic\right) &\le e^{-\frac{|\Ic|\delta^2}{4}}\\
\Pr\left( \frac{1}{|\Ic|}\sum_{(\Yc,y)\in\Ic} \1(y\in\messages_1) - p_*(1) \ge \frac{\delta}{4} \mid \Ic\right) &\le e^{-\frac{|\Ic|\delta^2}{8}}\\
\Pr\left(\frac{1}{|\Ic|}\sum_{(\Yc,y)\in\Ic} \1(y\in\messages_1) - p_*(1) \le -\frac{p_*(1)}{2} \mid \Ic \right) &\le e^{-\frac{|\Ic|p_*^2(1)}{2}}.
\end{align*}
A union bound gives
\begin{align*}
&\quad \Pr\left(\left\{\frac{1}{|\Ic|}\sum_{(\Yc,y)\in\Ic} \frac{|\Yc \cap \messages_1|}{|\Yc|} - \frac{1}{|\Ic|}\sum_{(\Yc,y)\in\Ic} \1(y\in\messages_1) > \frac{\delta}{4} \right\} \bigcap \left\{\frac{1}{|\Ic|}\sum_{(\Yc,y)\in\Ic} \1(y\in\messages_1) > 0\right\} \mid \Ic\right)\\
&> 1 - e^{-\frac{|\Ic|\delta^2}{4}} - e^{-\frac{|\Ic|\delta^2}{8}} - e^{-\frac{|\Ic|p_*^2(1)}{2}}.
\end{align*}
By Lemma \ref{le:high-probability-both-types}, there exists $0<\gamma<1$ such that $\Pr(|\Ic| \le |\data|(1-\gamma)) \le e^{-\frac{|\data| (1-\gamma)\gamma^2}{8}}$.  By Equation \eqref{eq:difference},
\begin{align*}
&\quad \Pr\left(\left\{\frac{|\data|}{|\Ic|}(\rho_\mathrm{data} - \rho_\mathrm{chosen}) > \frac{\delta}{4} \right\} \bigcap \left\{\frac{|\data|}{|\Ic|} \rho_\mathrm{chosen} > 0\right\} \bigcap \left\{|\Ic| > |\data|(1-\gamma))\right\}\right)\\
&= \Pr\left(\left\{\frac{|\data|}{|\Ic|}(\rho_\mathrm{data} - \rho_\mathrm{chosen}) > \frac{\delta}{4} \right\} \bigcap \left\{\frac{|\data|}{|\Ic|} \rho_\mathrm{chosen} > 0\right\} \mid |\Ic| > |\data|(1-\gamma))\right) \Pr\left(|\Ic| > |\data|(1-\gamma))\right)\\
&> \left(1 - e^{-\frac{|\data|(1-\gamma)\delta^2}{4}} - e^{-\frac{|\data|(1-\gamma)\delta^2}{8}} - e^{-\frac{|\data|(1-\gamma)p_*^2(1)}{2}}\right) \left(1 - e^{-\frac{|\data| (1-\gamma)\gamma^2}{8}}\right)\\
&> 1 - e^{-\frac{|\data|(1-\gamma)\delta^2}{4}} - e^{-\frac{|\data|(1-\gamma)\delta^2}{8}} - e^{-\frac{|\data|(1-\gamma)p_*^2(1)}{2}} - e^{-\frac{|\data| (1-\gamma)\gamma^2}{8}}.
\end{align*}
Let $\eta = \frac{(1-\gamma)\delta}{4} > 0$. It follows that
\begin{align*}
\Pr\left(\Ec_\eta(\data)\right) &= \Pr\left(\left\{ \rho_\mathrm{data} - \rho_\mathrm{chosen} > \eta \right\} \bigcap \left\{\rho_\mathrm{chosen} > 0\right\}\right)\\
&\ge \Pr\left(\left\{\frac{|\data|}{|\Ic|}(\rho_\mathrm{data} - \rho_\mathrm{chosen}) > \frac{\delta}{4} \right\} \bigcap \left\{\frac{|\data|}{|\Ic|} \rho_\mathrm{chosen} > 0\right\} \bigcap \left\{|\Ic| > |\data|(1-\gamma))\right\}\right)\\
&> 1 - e^{-\frac{|\data|(1-\gamma)\delta^2}{4}} - e^{-\frac{|\data|(1-\gamma)\delta^2}{8}} - e^{-\frac{|\data|(1-\gamma)p_*^2(1)}{2}} - e^{-\frac{|\data| (1-\gamma)\gamma^2}{8}}.
\end{align*}
The proof follows by taking $|\data|\to\infty$.

\end{proof}

\noindent
We are now ready to prove the following failure case for RLPO.  

\dichotomyRlpo*
\begin{proof}
The proof follows from applying Lemmas \ref{le:deterministic-failure} and \ref{le:high-probability-event}.
\end{proof}

\subsection{DPO}

In Lemma \ref{lem:dpo-equiv-rlpo}, we show that DPO and RLPO produces the same policy, which then allows us to directly apply the analysis for RLPO to reach the same conclusion for DPO. 

\begin{lemma}
\label{lem:dpo-equiv-rlpo}
If there exists a minimizer $\hat{r}$ to $\Lc_\mathrm{reward}$, then the optimal policy for $\Lc_\mathrm{DPO}$ and for $\Lc_\mathrm{policy}(\cdot|\hat{r})$ are identical.  
\end{lemma}

\begin{proof}
Recall that for any given reward function $r$, the optimal policy $\pi_r$ for $\Lc_\mathrm{policy}$ satisfies 
\begin{equation}
P_{\pi_r}(x) = \frac{P_{\overline{\pi}}(x) e^{r(x) |\data|/\beta}}{\sum_{x' \in \Mc} P_{\overline{\pi}}(x') e^{r(x') |\data|/ \beta}}.
\end{equation}
It follows that
\begin{align}
    - \sum_{(\Yc,y)\in\data} \left[ \ln \left(\frac{e^{r(y)}}{\sum_{y'\in\Yc} e^{r(y')}}\right)\right] = - \sum_{(\Yc,y)\in\data} \left[ \ln \frac{
\left(P_{\pi_r}(y)/ P_{\overline{\pi}}(y)\right)^{\beta/|\data|}}{\sum_{y'\in\Yc} 
\left(P_{\pi_r}(y') / P_{\overline{\pi}}(y')\right)^{\beta/|\data|}}\right].
\end{align}
Thus the value $\Lc_\mathrm{reward}(r)$ equals the value $\Lc_\mathrm{DPO}(\pi_r)$.  Note also that $r\mapsto \pi_r$ is surjective.  

Suppose $\hat{r}$ is the minimizer for $\Lc_\mathrm{reward}$.  Then $\pi_{\hat{r}}$ is the optimal policy obtained by minimizing $\Lc_\mathrm{policy}(\cdot|\hat{r})$.  If $\pi_{\hat{r}}$ is not optimal for $\Lc_\mathrm{DPO}$, then there exists another policy $\pi'$ that obtains a strictly lower loss $\Lc_\mathrm{DPO}(\pi')$.  Since $r\mapsto \pi_r$ is surjective, there exists a reward function $r'$ such that $\pi' = \pi_{r'}$.  For example, we can take $r'(x) = \frac{\beta}{|\data|}\ln\frac{P_{\pi'}(x)}{P_{\overline{\pi}}(x)}$, which achieves a lower reward loss $\Lc_\mathrm{reward}(r') = \Lc_\mathrm{DPO}(\pi_{r'}) < \Lc_\mathrm{DPO}(\pi_{\hat{r}}) = \Lc_\mathrm{reward}(\hat{r})$, a contradiction.  

Similarly, if $\pi^*$ is optimal for $\Lc_\mathrm{DPO}$, the corresponding reward $r^*(x) = \frac{\beta}{|\data|}\ln\frac{P_{\pi^*}(x)}{P_{\overline{\pi}}(x)}$ must be optimal for $\Lc_\mathrm{reward}$.  The corresponding optimal policy $\pi_{r^*}$ for $\Lc_\mathrm{policy}(\cdot|r^*)$ then satisfies 
$$P_{\pi_{r^*}}(x) \propto P_{\overline{\pi}}(x) e^{\frac{\beta}{|\data|}\ln\frac{P_{\pi^*}(x)}{P_{\overline{\pi}}(x)} |\data|/\beta} = P_{\pi^*}(x),$$
as desired.

\end{proof}

\dichotomyDpo*
\begin{proof}
By Lemma \ref{lem:dpo-equiv-rlpo}, if $\hat{\psi}$ is as defined in Proposition~\ref{prop:dichotomy-rlpo-policy}, then $\hat{\psi} = \argmin_\psi \Lc_\mathrm{reward}(\hat{r}_\psi)$ and we have $\hat{\theta} = \argmin_\theta \Lc_\mathrm{DPO}(\hat{\pi}_\theta) = \argmin_\theta \Lc_\mathrm{policy}(\hat{\pi}_\theta | \hat{r}_{\hat\psi})$.  The proof follows from Proposition \ref{prop:dichotomy-rlpo-policy}.
\end{proof}

\subsection{Inclusive Learning}

In this subsection, we prove Proposition \ref{prop:dichotomy-inclusive}.  We start by defining two key sets that describe messages in $\data$.  Recall that in the proof for Proposition~\ref{prop:dichotomy-rlpo-policy}, we defined
$$
\Ic = \big \{(\Yc, y) \in\data \, | \, \Yc \cap \messages_1 \ne \emptyset, \, \Yc\cap\messages_2\ne\emptyset \big \}.
$$
To facilitate the exposition, we also define
$$
\Ic_\tau = \left\{(\Yc,y)\in\Ic \mid y\in\messages_\tau\right\} \quad\text{for } \tau=1,2
$$
to be the set of data in $\Ic$ where a message of category $\tau$ is chosen.  We first prove the following lemma that characterizes the optimal solution for a class of convex loss functions.
\begin{lemma}
\label{le:convex-solution}
Suppose that for all $m>0$, the loss $\Lc(x;m):(0,1)\to\Re$ is convex in $x$, $\Lc(x_m;m) = \min_x \Lc(x;m)$, and there exists $M>0$ such that for all $m>M$, 
$$
\frac{\partial \Lc(x;m)}{\partial x} < 0.
$$
Then $\lim_{m\to\infty}x_m = 1$.
\end{lemma}
\begin{proof}
Since $\Lc(x;m)$ is convex on $(0,1)$, its local minima is also a global minima.  For all $m > M$, $\Lc$ is strictly decreasing for $x\in(0,1)$, so $\lim_{m\to\infty} x_m = 1$.  
\end{proof}

\noindent
Then, we define the high probability event
$$
\Ec(\data) = \left\{|\Ic_1| > 1\right\} \cap \left\{|\Ic_2| > \max\{1+\beta, 2\beta\ln\frac{p_*(2)}{p_*(1)}\}\right\}.
$$

\begin{lemma}
\label{le:deterministic-failure-IL}
Under Assumption \ref{as:dichotomy-policy-architecture}, for all fixed $\messages_1$ and $\data$ such that $\Ec(\data)$ holds, if $\hat{\theta}$ minimizes $\Lc_\mathrm{IL}(\hat{\pi}_\theta)$, then
$$
P_{\hat{\pi}_{\hat{\theta}}}(\Mc_1) \rightarrow 0 \qquad \text{and} \qquad \nn P_{\hat{\pi}_{\hat{\theta}}}(\Mc_2) \rightarrow 1
$$
as the size of $\Mc_2$ grows.
\end{lemma}
\begin{proof}
First, we notice that $\Lc_\mathrm{IL}(\hat{\pi}_\theta)$ is strongly convex for $\beta > 0$ and $\Lc_\mathrm{IL}(\hat{\pi}_\theta) \to \infty$ as $\|\theta\|_2\to\infty$, thus a minimizer for $\Lc_\mathrm{IL}(\hat{\pi}_\theta)$ exists and is unique.  Let $m = \frac{|\messages_2|}{|\messages_1|}$ and $q = P_{\hat{\pi}_\theta}(\messages_2)$.  The loss function can be written in terms of $m$ and $q$ as
\begin{equation}
\Lc(q,m) = \Lc_\mathrm{IL}(\hat{\pi}_\theta) = - \sum_{(\Yc,y)\in\data} \ln \frac{m(1-q)\cdot \1_{\messages_1}(y) + q\cdot\1_{\messages_2}(y)}{m(1-q)\cdot|\Yc\cap\messages_1| + q\cdot|\Yc\cap\messages_2|} + \beta\left[(1-q) \ln\frac{1-q}{p_*(1)} + q\ln\frac{q}{p_*(2)}\right].
\end{equation}
Taking the partial derivative with respect to $q$, under the event $\Ec(\data)$, 
\begin{align*}
\frac{\partial \Lc(q,m)}{\partial q} &= \sum_{(\Yc,y)\in\Ic} \frac{|\Yc\cap\messages_2|}{m(1-q)^2|\Yc\cap\messages_1| + q(1-q)|\Yc\cap\messages_2|} - \sum_{(\Yc,y)\in\Ic_2}\left( \frac{1}{q} + \frac{1}{1-q}\right) + \beta\left(\ln\frac{q}{1-q} + \ln\frac{p_*(1)}{p_*(2)}\right)\\
&\stackrel{(a)}{<} \sum_{(\Yc,y)\in\Ic} \frac{|\Yc\cap\messages_2|}{m(1-q)^2 |\Yc\cap\messages_1| + q(1-q)|\Yc\cap\messages_2|} - \frac{|\Ic_2|}{q} - \frac{|\Ic_2|}{1-q} + \frac{\beta}{1-q} + \frac{|\Ic_2|}{2}\\
&\stackrel{(b)}{<} \sum_{(\Yc,y)\in\Ic} \frac{|\Yc\cap\messages_2|}{m(1-q)^2 |\Yc\cap\messages_1| + q(1-q)|\Yc\cap\messages_2|} - \frac{|\Ic_2|}{2q} - \frac{1}{1-q},
\end{align*}
where $(a)$ follows from $|\Ic_2| > 2\beta\ln\frac{p_*(2)}{p_*(1)}$ and $\ln\frac{q}{1-q} < \frac{1}{1-q}$, and $(b)$ follows from $|\Ic_2|>1+\beta$ and $q<1$.  For all $0<q<1$, the first term is always positive and converges to 0 as $m\to\infty$.  Thus, there exists an $M>0$ such that for all $m>M$, $\sum_{(\Yc,y)\in\Ic} \frac{|\Yc\cap\messages_2|}{m(1-q)^2 |\Yc\cap\messages_1| + q(1-q)|\Yc\cap\messages_2|} < \frac{|\Ic_2|}{2q}$.  This implies that 
$$
\frac{\partial \Lc(q,m)}{\partial q} < -\frac{1}{1-q} < 0.
$$
The proof follows from applying Lemma \ref{le:convex-solution}.  

\end{proof}

\noindent
Finally, we prove that under the dichotomy data assumption, the event $\Ec(\data)$ holds with high probability. 
\begin{lemma}
\label{le:high-probability-event-IL}
Under Assumption \ref{as:dichotomy}, if $|\Yc| \ge 2$, then for fixed $\messages_1$, $\Pr(\Ec(\data))\to 1$ as $|\messages_2|\to\infty$. 
\end{lemma}
\begin{proof}
By Lemma \ref{le:high-probability-both-types}, there exists a constant $\gamma\in(0,1)$ such that for all $|\data|>0$, $\Pr(|\Ic| \le |\data|(1-\gamma)) \le e^{-\frac{|\data| (1-\gamma)\gamma^2}{8}}$.  Let $\epsilon(\beta) = \max\{\beta+1, 2\beta\ln\frac{p_*(1)}{p_*(2)}\}$, then $\E[\1(y\in\messages_2) \mid (\Yc,y)\in\Ic] = p_*(2)$.  By Hoeffding's inequality,
\begin{align*}
\Pr\left(\sum_{(\Yc,y)\in\Ic} \1(y\in\messages_2) - |\Ic|\cdot p_*(2) \le \epsilon(\beta) - |\Ic|\cdot p_*(2) \mid \Ic \right) &\le e^{-2\frac{(\epsilon(\beta) - |\Ic|p_*(2))^2}{|\Ic|}} < e^{-2|\Ic|p_*^2(2) + 4\epsilon(\beta) p_*(2)}\\
\Pr\left(\sum_{(\Yc,y)\in\Ic} \1(y\in\messages_1) - |\Ic|\cdot p_*(1) \le 1 - |\Ic|\cdot p_*(1) \mid \Ic \right) &\le e^{-2\frac{(1 - |\Ic|p_*(2))^2}{|\Ic|}} < e^{-2|\Ic|p_*^2(2) + 4 p_*(2)}.
\end{align*}
Following a similar argument as that in the proof for Lemma \ref{le:high-probability-event}, we have
\begin{align*}
\Pr(\Ec(\data)) &\ge \Pr\left(\left\{|\Ic| > |\data|(1-\gamma)\right\} \cap \left\{|\Ic_1| > 1\right\} \cap \left\{|\Ic_2| > \epsilon(\beta)\right\}\right)\\ 
&\ge 1 - e^{-\frac{|\data| (1-\gamma)\gamma^2}{8}} - e^{-2|\data|(1-\gamma)p_*^2(2) + 4\epsilon(\beta) p_*(2)} - e^{-2|\data| (1-\gamma) p_*^2(2) + 4 p_*(2)}.
\end{align*}
The proof follows by taking $|\messages_2|\to\infty$.

\end{proof}

\noindent
Finally, we prove the following failure case for inclusive learning.
\dichotomyInclusive*
\begin{proof}
The proof follows from applying Lemmas \ref{le:deterministic-failure-IL} and \ref{le:high-probability-event-IL}, and noticing that $\Pr( \lim_{|\messages_2|\to\infty} P_{\hat{\pi}_{\hat\theta}}(\messages_1) =0) \ge \Pr(\Ec(\data))$.  
\end{proof}

\subsection{SLiC}

The proof for Proposition \ref{prop:dichotomy-slic} is similar to that for Proposition \ref{prop:dichotomy-inclusive}.  We present the proof here. 

\begin{lemma}\label{le:deterministic-failure-Slic}
    Under Assumption \ref{as:dichotomy-policy-architecture}, if $|\Yc| = 2$, for all fixed $\messages_1$ and $\data$ such that $\Ec(\data)$ holds, if $\hat\theta$ minimizes $\Lc_\mathrm{SLiC}(\hat{\pi}_\theta)$, then
    $$
    P_{\hat{\pi}_{\hat{\theta}}}(\Mc_1) \rightarrow 0 \qquad \text{and} \qquad \nn P_{\hat{\pi}_{\hat{\theta}}}(\Mc_2) \rightarrow 1
    $$
as the size of $\Mc_2$ grows.
\end{lemma}
 \begin{proof}
     First, we notice that $\Lc_\mathrm{SLiC}(\hat{\pi}_\theta)$ is strongly convex for $\beta > 0$ and $\Lc_\mathrm{SLiC}(\hat{\pi}_\theta) \to \infty$ as $\|\theta\|_2\to\infty$, thus a minimizer for $\Lc_\mathrm{SLiC}(\hat{\pi}_\theta)$ exists and is unique.  Let $m = \frac{|\messages_2|}{|\messages_1|}$ and $q = P_{\hat{\pi}_\theta} (\messages_2)$.  When $|\Yc|=2$, the loss function can be written in terms of $m$ and $q$ as
     \begin{align*}
         \Lc(q,m) = \Lc_\mathrm{SLiC}(\hat{\pi}_\theta) &= \sum_{(\Yc,y)\in\data} \left[\ln \left(\frac{m(1-q)}{q} \1_{\messages_1}(y) \cdot \1_{\messages_2}(\Yc\setminus\{y\}) + \frac{q}{m(1-q)} \1_{\messages_2}(y) \cdot \1_{\messages_1}(\Yc\setminus\{y\}) \right) + \delta \right]_+\\ 
         & \quad + \beta \left[(1-q)\ln\frac{1-q}{p_*(1)} + q\ln\frac{q}{p_*(2)}\right]. 
     \end{align*}
     Let $h(q,m) = \ln \frac{m(1-q)\cdot \1_{\messages_1}(y) + q\cdot \1_{\messages_2}(y)}{q\cdot \1_{\messages_2}(\Yc\setminus\{y\}) + m(1-q)\cdot \1_{\messages_1}(\Yc\setminus\{y\})} + \delta$.  We have that
     \begin{align*}
         \frac{\partial h(q,m)}{\partial q} &= -\frac{m}{q^2} \1_{\messages_1}(y) \cdot \1_{\messages_2}(\Yc\setminus\{y\}) + \frac{1}{m(1-q)^2} \1_{\messages_2}(y) \cdot \1_{\messages_1}(\Yc\setminus\{y\}).
     \end{align*}
     For any fixed $q\in(0,1)$, when $m$ is large enough, $\frac{m(1-q)}{q} > e^{-\delta}$ and $\frac{q}{m(1-q)} < e^{- \delta}$.  Consequently, when $m$ is large enough, 
     \begin{align*}
         \frac{\partial \Lc(q,m)}{\partial q} &= - \sum_{(\Yc,y) \in \Ic_1} \frac{m}{q^2} + \beta\left((1-q)\ln\frac{1-q}{p_*(1)} + q\ln\frac{q}{p_*(2)} \right) < 0.
     \end{align*}
     The proof follows by applying Lemma \ref{le:convex-solution}.  
 \end{proof}

With this, we are ready to prove the failure case for SLiC-HF. 
\dichotomySlic*
\begin{proof}
    The proof follows from applying Lemmas \ref{le:deterministic-failure-Slic} and \ref{le:high-probability-event-IL}, and noticing that $\Pr( \lim_{|\messages_2|\to\infty} P_{\hat{\pi}_{\hat\theta}}(\messages_1) =0) \ge \Pr(\Ec(\data))$.  
\end{proof}

\newpage
\bibliography{references}

\end{document}